\documentclass[twoside]{article}

\usepackage[accepted]{aistats2025}
\usepackage[round]{natbib}

\usepackage[utf8]{inputenc} 
\usepackage[T1]{fontenc}    
\usepackage{hyperref}       
\usepackage{url}            
\usepackage{booktabs}       
\usepackage{amsfonts}       
\usepackage{nicefrac}       
\usepackage{microtype}      

\usepackage{graphicx}
\usepackage{booktabs,colortbl} 
\usepackage{multirow}
\usepackage{fixltx2e}
\usepackage{wrapfig}
\usepackage{caption}
\usepackage{subcaption}
\usepackage{algorithm}
\usepackage{algpseudocode}

\usepackage[dvipsnames]{xcolor}

\newcommand{\grn}[1]{{\color{Green} \textbf{#1}}}

\usepackage[textsize=tiny]{todonotes}

\usepackage{amsmath}
\usepackage{amssymb}
\usepackage{mathtools}
\usepackage{amsthm}
\usepackage[mathscr]{eucal}

\usepackage{mdwlist}

\theoremstyle{plain}
\newtheorem{theorem}{Theorem}[section]

\newtheorem{lemma}[theorem]{Lemma}

\theoremstyle{definition}
\newtheorem{definition}[theorem]{Definition}

\newcommand{\eq}[1]{\begin{equation}{#1}\end{equation}}
\newcommand{\ml}[1]{\begin{multline}{#1}\end{multline}}
\newcommand{\al}[1]{\begin{align}{#1}\end{align}}
\newcommand{\prn}[1]{\left({#1}\right)}
\newcommand{\brt}[1]{\left[{#1}\right]}
\newcommand{\abs}[1]{\left|{#1}\right|}

\newcommand{\mbf}[1]{\mathbf{#1}}
\newcommand{\mbb}[1]{\mathbb{#1}}
\newcommand{\mc}[1]{\mathcal{#1}}
\newcommand{\ms}[1]{\mathscr{#1}}
\newcommand{\bs}[1]{\boldsymbol{#1}}
\newcommand{\nn}{\nonumber}

\begin{document}

%

%

\twocolumn[

\aistatstitle{Understanding GNNs and Homophily in Dynamic Node Classification}

\aistatsauthor{ Michael Ito \And Danai Koutra \And Jenna Wiens }

\aistatsaddress{ University of Michigan \And University of Michigan \And University of Michigan }

]

\begin{abstract}
Homophily, as a measure, has been critical to increasing our understanding of graph neural networks (GNNs). However, to date this measure has only been analyzed in the context of static graphs. In our work, we explore homophily in dynamic settings. Focusing on graph convolutional networks (GCNs), we demonstrate theoretically that in dynamic settings, current GCN discriminative performance is characterized by the probability that a node's \textit{future} label is the same as its neighbors' \textit{current} labels. Based on this insight, we propose \textit{dynamic homophily}, a new measure of homophily that applies in the dynamic setting. This new measure correlates with GNN discriminative performance and sheds light on how to potentially design more powerful GNNs for dynamic graphs. Leveraging a variety of dynamic node classification datasets, we demonstrate that popular GNNs are not robust to low dynamic homophily. Going forward, our work represents an important step towards understanding homophily and GNN performance in dynamic node classification.
\end{abstract}

\section{INTRODUCTION}

In node classification tasks, graph neural networks (GNNs) obtain strong performance on highly homophilous graphs, where most edges connect nodes of similar labels. In contrast, on many heterophilous graphs where most edges connect nodes of opposing labels, GNN performance degrades \citep{Pei2020Geom-GCN}. Measuring homophily of a graph-based task has since been used in attempts to characterize the discriminative power of GNNs and the types of graphs for which GNN performance is limited \citep{zhu2020beyond}. Several works have focused on understanding the relationship between GNNs and homophily, both at the graph and node level~\citep{Pei2020Geom-GCN,LovelandK25fairness,ma2022is,ZhuLY0CK24linkpred,luan2024heterophilicgraphlearninghandbook}. These analyses have inspired key GNN designs that provide good generalization performance on homophilous and heterophilous graphs in static node classification~\citep{chien2021adaptive, bo2021beyond, zhu2021graph, yan2022two, ItoKW25llpe}. Homophily has thus been crucial in advancing the study of GNNs since it has increased our understanding of GNN limitations and in turn led to the development of new methods to overcome these limitations.


However, to date, the vast majority of works analyzing GNNs and homophily in node classification have assumed a static graph with unchanging features, labels, and structure \citep{platonov2024characterizing}. In many real-world node classification tasks, the graph changes over time. For example, consider a node classification problem from epidemiology where the goal is to predict the spread of a contagion on a temporal contact network of individuals. Here, node features change across time due to time-varying individual characteristics, node labels change across time due to the spread of the contagion, and the graph structure changes across time due to changes in contact patterns of the individuals.

Inspired by this problem, we explore notions of homophily and GNNs in the context of node classification on dynamic graphs. We theoretically analyze graph convolutional networks (GCNs), a widely used GNN variant, and demonstrate that in dynamic settings, GCN discriminative performance is characterized by the probability that node future labels are the same as their neighbors' current labels. Based on this finding, we propose \textit{dynamic homophily}, a new homophily measure that is highly correlated with the discriminative performance of GCNs in dynamic settings. Our theory further suggests potential designs for more powerful dynamic GNNs that can overcome low dynamic homophily settings. Empirically, we apply dynamic homophily to dynamic node classification datasets from epidemiology, social network analysis, and molecular biology, demonstrating that our theoretical analyses hold in real-world dynamic graphs for a variety of GNNs. In summary, we make the following contributions.

\begin{itemize}
    \item \textbf{Theoretical Analysis.} We present a theoretical analysis showing that in dynamic settings both the separation and variances of node representations produced by a GCN can be represented as a function of the probability that a node's future label is the same as their neighbors' current labels. As a result, GCN discriminative performance can be represented as a function of this probability. Our analysis sheds light on how to potentially design more powerful dynamic GNNs robust to low dynamic homophily settings.
    
    \item \textbf{New Homophily Definition for Dynamic Graphs.} Based on our theoretical findings, we propose a novel definition of homophily for the dynamic setting called dynamic homophily.

    \item \textbf{Real-world \& Synthetic Experiments.} Applied to dynamic node classification tasks, we show that dynamic homophily accurately correlates with the discriminative performance of many GNNs.
\end{itemize}

\section{PRELIMINARIES}
We first introduce notation and define node classification on a static graph. We then provide an overview of message passing, a common framework for node classification, focusing on linear GCNs, a key tool in our theoretical analysis. We next provide a background on the homophily measures on which we build and formalize our problem setting of node classification on dynamic graphs.

\subsection{Notation}
We define the static graph $G = (V, \mbf{A}, \mbf{X}, \mbf{y})$, where $V$ is set of nodes, $\mbf{A} \in \{0, 1\}^{\abs{V} \times \abs{V}}$ is the adjacency matrix, $\mbf{X} \in \mbb{R}^{\abs{V} \times d}$ is the node feature matrix, and $\mbf{y} \in \mbb{R}^{\abs{V}}$ is the vector of node labels. Let $C$ be the set of nodes classes. Let $\mbf{d} \in \mbb{N}^{\abs{V}}$ be the vector of node degrees. For node $i \in V$, we denote the node feature vector and node label as $\mbf{x}(i)$ and $y(i)$, respectively. We denote its degree and its set of one-hop neighbors as $d(i)$ and $\mc{N}(i)$, respectively. Let $\hat{\mc{N}}(i)$ be the set of one-hop neighbors after the addition of a self-loop. 

\subsection{Node Classification on a Static Graph}
Let $G$ be a static graph with adjacency matrix $\mbf{A}$ from the space of adjacency matrices $\mc{A}$. Given a random sample of node representations $\mbf{X}_\text{train} = \{\mbf{x}(0), \ldots, \mbf{x}(n_{\text{train}})\}$ from input space $\mc{X}$, and their labels $\mbf{y}_\text{train} = \{y(0), \ldots, y(n_{\text{train}})\}$ from output space $\mc{Y}$, the node classification task on a static graph is to learn a classifier $f: \mc{X} \times \mc{A} \to \mc{Y}$ such that the expected loss $\mbb{E}[\ms{L}(f(\mbf{x}, \mbf{A}), y)]$ over the training set is minimized, where $\ms{L}$ is some loss function. Once learned, $f$ can be used to label the remaining nodes in the graph.

\subsection{Message-Passing GNNs and Homophily}
GNNs leverage the graph structure $\mbf{A}$ and node feature matrix $\mbf{X}$ to learn new representations $\mbf{H} \in \mbb{R}^{\abs{V} \times h}$. Most GNNs follow a message-passing scheme, where each GNN layer updates each node's representation using the representations of its immediate neighbors \citep{gilmer2017neural}. Formally, the $l$-th layer of a GNN can be summarized at the node level as the following propagation rule for all $i \in V$,
\eq{\mbf{h}^{(l+1)}(i) = \textsc{Aggregate}^{(l)}(\{\mbf{h}^{(l)}(j) : j \in \hat{ \mc{N}}(i)\}),}

where $\mbf{h}^{(l)}(j)$ is the representation for node $j$ at layer $l$ of a GNN and $\textsc{Aggregate}$ is a function that treats $\hat{\mc{N}}(i)$ as a set of nodes. In our work, we focus on linear GCNs which are GNNs that leverage mean aggregation determined by the propagation rule:
\eq{\mbf{h}^{(l+1)}(i) = \frac{1}{d(i) + 1}\sum_{j \in \hat{\mc{N}}(i)} \mbf{h}^{(l)}(j) \label{eq:2}.}
\vspace{-1em}

GCNs are widely used due to their high performance in many tasks \citep{kipf2017semi, wu2019simplifying}. Linear GCNs have also been rigorously analyzed due to their simplicity, and various limitations of message passing GNNs have been discovered, such as oversmoothing and the heterophily problem \citep{li2018deeper, zhu2020beyond}. Similarly, in our analysis, we leverage linear GCNs as a key tool in demonstrating the limitations of GNNs in dynamic node classification. 

Homophily is the probability that a node forms an edge with another node with the same label \cite{zhu2020beyond}. Intuitively, when homophily is high, message passing is beneficial since different classes will be well separated in feature space after message passing. Formally, given a static graph $G$, static homophily measured with respect to all nodes in $V$ is defined as: 
\eq{h^S = \mbb{P}(y(i) = y(j) \mid j \in \mc{N}(i)).}

\begin{figure*}[t!]
    \centering
    \includegraphics[width=\textwidth]{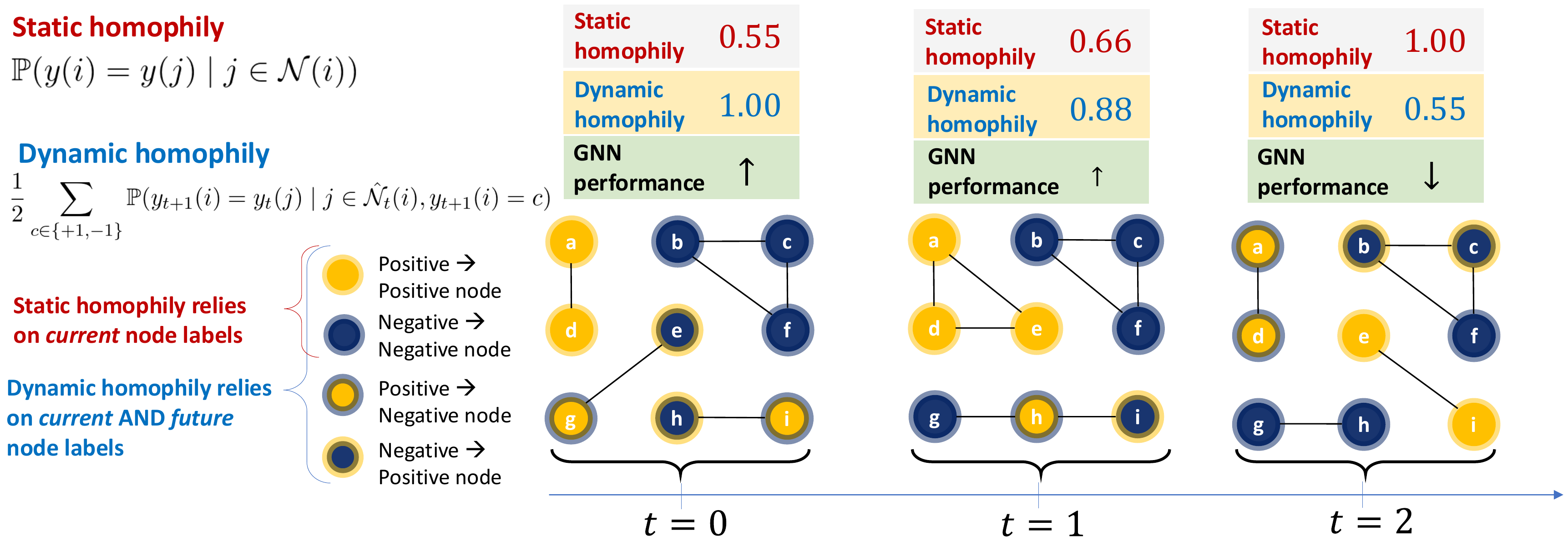}
    \caption{Toy dynamic graph where the task is to predict future node labels. When $t=0$, GNNs obtain good performance in predicting future node labels. Dynamic homophily is high since all nodes have the same \textit{future label} as their neighbors' current label while, static homophily is low since many nodes (e, g, h, i) and their neighbors have different labels at the \textit{current timestep}. When $t=2$, GNNs obtain poor performance. Dynamic homophily is low since many nodes (a, b, c, d) have a different \textit{future label} than their neighbors' current label, while static homophily is high since all nodes and their neighbors have the same current label at $t=2$.}
    \label{fig:1}
\end{figure*}

\subsection{Our Problem Setting: Node Classification on a Dynamic Graph}
A dynamic graph $\bs{G}_{0:T}$ is defined as a sequence of static graphs $\bs{G}_{0:T} = \{G_0, \ldots, G_T\}$, where $T+1$ is the number of total timesteps. The static graph at time $t$ is defined as $G_t = (V_t, \mbf{A}_t, \mbf{X}_t, \mbf{y}_t)$ where $V_t$ is the set of nodes at time $t$, $\mbf{A}_t \in \{0, 1\}^{\abs{V} \times \abs{V}}$ is the adjacency matrix at time $t$, $\mbf{X}_t \in \mbb{R}^{\abs{V_t} \times d}$ is the feature matrix at time $t$, and $\mbf{y}_t \in \mbb{R}^{\abs{V} }$ is the label vector at time $t$. Given training data $\{G_{0:T}^{(k)}\}_{k=1}^{n_{\text{train}}}$, the goal of node classification on dynamic graphs is to learn a classifier $f: \mc{X}_t \times \mc{A}_t \times \mc{Y}_t \to \mc{Y}_{t+1}$ such that the expected loss $\mbb{E}[\ms{L}(f(x_t, \mbf{A}_t, y_t), y_{t+1})]$ is minimized for all $t \in [0, T]$. $f$ can then be applied to a new dynamic graph, with the goal of predicting labels at the next time step.

\section{THEORETICAL ANALYSIS}

We first introduce our setup, outlining our framework for characterizing GCN discriminative performance in the dynamic setting. We then present our main results, characterizing GCN discriminative performance in the dynamic setting and providing potential insights on how to design more powerful dynamic GNNs. Based on our theoretical results, we introduce a new measure of homophily, dynamic homophily, that correlates with GNN performance in the dynamic setting. We lastly provide an extensive comparison between our results in the dynamic setting and existing results in the static setting. In Figure \ref{fig:1}, we provide an overview of our definitions, building intuition and highlighting the differences between static and dynamic homophily.

\subsection{Setup}

In this section, we show that the expected distance between nodes of different classes and the variance of nodes at time $t$ characterize GCN discriminative performance at time $t$. This intermediate result provides a framework for our theoretical analysis. Specifically, by measuring the expected distance and the variance of the node representations after GCN layers at time $t$, we can quantify GCN discriminative performance at time $t$. We utilize the following assumptions. 

\textbf{Assumptions.} Let $\bs{G}_{0:T}$ be a dynamic graph. For all $t \in [0, T]$ and for all $i \in V_t$, $y_t(i) \in \{-1, +1\}$, and $x_t(i) \sim N(y_t(i) \cdot \mu_t, \sigma^2_t) \in \mbb{R}$ where $N$ is the normal distribution. We do not assume any specific temporal process on the labels of $\bs{G}_{0:T}$, allowing our results to generalize across many dynamic settings. We also do not make any explicit assumptions on the node homophily distribution. We assume $f^{(l)}$ is a linear GCN. While our analysis relies on these typical assumptions  \citep{wu2019simplifying, zhu2020beyond} for tractability, we empirically verify our claims when these assumptions do not hold.

Since the task at each time step is binary, we utilize the AUROC as a metric for discriminative power in favor of other metrics such as the misclassification rate which is affected by class imbalances. The AUROC of a GCN is the probability that it ranks a random positive node higher than a random negative node. Thus, it accurately measures a GCN's ability to discriminate between the two classes. The following lemma states that we can explicitly characterize the AUROC of a linear GCN applied at time $t$ in terms of the expected distance between node representations of opposing classes and their variances at time $t$. 

\begin{lemma}
The expected AUROC of $f^{(l)}$ at timestep $t$ can be written as follows, 
\eq{\resizebox{1.0\columnwidth}{!}{ $\mbb{E}[A_t(f^{(l)})] = 1 - \Phi\prn{-\frac{\mbb{E}_{i | i \in V_{t+1}^+}[\mbf{h}^{(l)}_{t}(i)] -\mbb{E}_{j | j \in V_{t+1}^-}[\mbf{h}^{(l)}_{t}(j)]}{\mbb{V}_{i, j| i \in V_{t+1}^+, j \in V_{t+1}^-}[\mbf{h}^{(l)}_{t}(i)  + \mbf{h}^{(l)}_{t}(j)] }}$ \label{eq:lem}}}
where $\Phi$ is the cumulative distribution function of the Gaussian distribution, and $V_{t+1}^+$ and $V_{t+1}^-$ are positive and negative nodes at time $t+1$, respectively.
\label{lem:1}
\end{lemma}

We prove Lemma \ref{lem:1} in Appendix \ref{pf:lem1}. Lemma \ref{lem:1} tells us that the expected AUROC at time $t$ is monotonically increasing in the ratio of expected distance over the variances of the nodes at time $t$. Intuitively, an increase in the expected distance increases the distance \textit{between} the two classes, while a decrease in the variances decreases the distance \textit{within} the two classes. 

\subsection{Main Results} 

To characterize GCN discriminative performance, we first measure the expected distance in node representations between the positive and negative classes after $l$ GCN layers, showing that the distance is characterized by the probability that node future labels are the same as their neighbors' current labels. We next show that node representation variances of the positive and negative classes after $l$ GCN layers are also characterized by this probability. Based on these findings, we propose dynamic homophily, a new measure of homophily for dynamic node classification that better reflects GNN discriminative power compared to static homophily. Our results further provide potential insights in how to design more powerful GNNs in the dynamic setting. In the following theorem, we present our first result measuring the expected difference in node representations after $l$ GCN layers.

\begin{theorem}
At time $t$, the difference in expected node representations between a future positive and negative node after $l$ layers of a GCN can be expressed as: 
\ml{\mbb{E}_{i | i \in V_{t+1}^+}[\mbf{h}^{(l)}_{t}(i)] -\mbb{E}_{j | j \in V_{t+1}^-}[\mbf{h}^{(l)}_{t}(j)] \\= 2 \cdot \mu_t \cdot (h_t^+ + h_t^- - 1)^l.}

where $h_{t}^+$ is the probability node $i$'s label at time $t+1$ is the same as its neighbor's label at time $t$ given node $i$'s label is positive at $t+1$ such that: 
{\small{\eq{h_{t}^+ = \mbb{P}(y_{t+1}(i) = y_{t}(j)\mid j \in \hat{\mc{N}}_t(i), y_{t+1}(i) = +1),}}}

and $h_{t}^-$ is the probability node $i$'s label at time $t+1$ is the same as its neighbor's label at time $t$ given node $i$'s label is negative at $t+1$ such that: 
{\small{\eq{h_{t}^- = \mbb{P}(y_{t+1}(i) = y_{t}(j)\mid j \in \hat{\mc{N}}_t(i), y_{t+1}(i)=-1).}}}

We denote $h_{t}^+$ and $h_{t}^-$ as the positive and negative class dynamic homophily levels, respectively.
\label{thm:1}
\end{theorem}

We prove Theorem \ref{thm:1} in Appendix \ref{pf:dist}. Before discussing the implications of Theorem \ref{thm:1}, we first discuss the intuition behind our new homophily measures. The main idea is to measure node neighbors' contributions at time $t$ towards node future labels at time $t + 1$. In this manner, the positive and negative class dynamic homophily levels accurately capture relationships across time. Theorem \ref{thm:1} tells us that the expected distance in node representations at time $t$ is a function of these homophily levels at time $t$. Specifically, the distance in node representations is \textit{polynomial} in the sum of the positive and negative class dynamic homophily levels. Thus, increases in the sum of these homophily levels increases the separation of node representations, and we expect the discriminative performance of GCNs to increase. 

While it is well understood that in the limit as $l \to \infty$ GCNs \textit{oversmooth}, and the difference in node representations becomes 0, rendering nodes indistinguishable \citep{li2018deeper}, Theorem \ref{thm:1} tells us that the rate of oversmoothing is a function of the dynamic homophily levels, thus bridging the gap between homophily and oversmoothing in the dynamic setting similar to analyses in the static setting \citep{bodnar2022neural, yan2022two}. We now demonstrate that the empirical distance concentrates around its expected distance.

\begin{theorem}
For $\epsilon > 0$, the probability that at time $t$ the distance between the empirical and expected distance after $l$ GCN layers is larger than $\epsilon$ is bounded as follows:
\eq{\resizebox{1.0\columnwidth}{!}{$\mbb{P}(\lvert (\bs{\mu_{V_{t+1}^+}^{(l)}} - \bs{\mu_{V_{t+1}^-}^{(l)}}) - (\mbb{E}_{i |i \in V_{t+1}^+}[\mbf{h}^{(l)}_{t}(i)] -\mbb{E}_{i | i \in V_{t+1}^-}[\mbf{h}^{(l)}_{t}(i)]) \rvert \geq \epsilon)$} \nn}
{\small{\eq{\hspace{10em} \leq \mc{O}(e^{-\epsilon^2 L_{t, +}^{(l)}} + e^{-\epsilon^2 L_{t,-}^{(l)}})}}},

where $\bs{\mu_{V_{t+1}^+}^{(l)}}$ and $\bs{\mu_{V_{t+1}^-}^{(l)}}$ are the empirical mean representations after $l$ GCN layers over future positive and negative nodes, respectively, and
{\small{\al{L_{t, +}^{(l)} &= \frac{ \abs{V_{t+1}^+}^2}{\sigma^4 \cdot \sum_{i \in V_{t+1}^+} \prn{\sum_{j \in \hat{\mc{N}}_t(i)} \frac{l}{d_t(j)^l}}^{2}} \\ L_{t, -}^{(l)} &= \frac{ \abs{V_{t+1}^-}^2}{\sigma^4 \cdot \sum_{i \in V_{t+1}^-} \prn{\sum_{j \in \hat{\mc{N}}_t(i)} \frac{l}{d_t(j)^l}}^{2}}.}}}
\label{thm:2}
\vspace{-1em}
\end{theorem}

\begin{figure*}[t!]
    \centering
    \includegraphics[width=\textwidth]{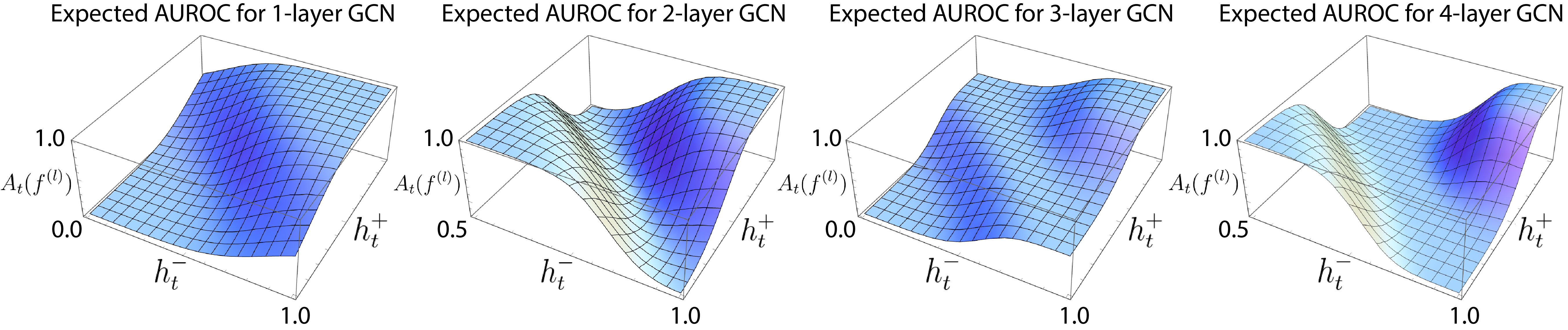}
    \caption{Expected AUROC across GCN layers as a function of dynamic homophily levels. The AUROC of odd layer GCNs is monotonically increasing in dynamic homophily levels, while the AUROC of even layer GCNs increase as both dynamic homophily levels approach 0 and 1, providing potential insights into how to design more powerful dynamic GNNs that can adapt to dynamic homophily levels across time.}
    \label{fig:auc}
\end{figure*}

We prove Theorem \ref{thm:2} in Appendix \ref{pf:con}. Theorem \ref{thm:2} tells us that the distance at time $t$ between the empirical means and the expected representations after $l$ GCN layers is small with high probability. In particular, this probability depends exponentially on $\abs{V_{t+1}^+}$, the number of positive nodes, $\abs{V_{t+1}^-}$, the number of negative nodes, and $d_t(i)$, the node degrees. If these quantities are large, the empirical means are close to their expectations with high probability, and the sum of the positive and negative dynamic homophily levels correlate with the empirical distance at time $t$. Our next theorem upper bounds GCN discriminative performance in terms of positive and negative class dynamic homophily levels.

\begin{theorem}
The expected AUROC of $f^{(l)}$ at timestep $t$ can be upper bounded as follows, 
\eq{\mbb{E}[A_t(f^{(l)})] \leq 1 - \Phi\prn{-\frac{2 \cdot \mu_t \cdot (h_t^+ + h_t^- - 1)^l}{v_{t+1}^+(l) + v_{t+1}^-(l)}}.}
where $v_{t+1}^+(l)$ and $v_{t+1}^-(l)$ are the lower bounds of the variances of the future positive and negative nodes after $l$ GCN layers, respectively, and are defined recursively in terms of the dynamic homophily levels as follows,
\al{v_{t+1}^+(l) &= {h_t^+}^2 \cdot v_{t+1}^+(l-1) + (1 - {h_t^+})^2 \cdot v_{t+1}^-(l-1) \nn \\
v_{t+1}^-(l) &= {h_t^-}^2 \cdot v_{t+1}^-(l-1) + (1-{h_t^-})^2 \cdot v_{t+1}^+(l-1) \nn \\
v_{t+1}^+(0) &= v_{t+1}^-(0) = \sigma^2_t.}
Moreover the probability that at time $t$ the distance between the empirical AUROC and the expected AUROC of a $l$ layer GCN is larger than $\epsilon$ is bounded as follows, 
{\small{\eq{\mbb{P}(\abs{A_t(f^{(l)}) - \mbb{E}[A_t(f^{(l)})]} \geq \epsilon) \leq e^{\frac{-2\cdot\abs{V_{t+1}^+}\cdot\abs{V_{t+1}^-}\cdot\epsilon^2}{\abs{V_{t+1}^+}+\abs{V_{t+1}^-}}}}}}.
\label{thm:3}
\vspace{-2em}
\end{theorem}

We prove Theorem \ref{thm:3} in Appendix \ref{pf:auc}. Theorem \ref{thm:3} tells us that for GCNs with an odd number of layers, the AUROC upper bound is monotonically increasing in the dynamic homophily levels for the positive and negative class, while for GCNs with an even number of layers the upper bound is monotonically increasing as both homophily levels approach 0 and 1. To demonstrate this intuition, consider 1 and 2-layer GCNs along binary stochastic block models (SBMs). When the positive and negative class homophily levels increase, more edges within communities are present, and a 1-layer GCN's AUROC can only increase. When both homophily levels are low, the SBM resembles a bipartite graph, and 2-layer GCNs are able to recover the correct representations. 

We illustrate Theorem \ref{thm:3} by visualizing the AUROC upper bound as a function of positive and negative dynamic class homophily levels for different GCN layers in Figure \ref{fig:auc}. Notice that we recover known mid-homophily pitfalls in the static setting \citep{luan2024graph} since the worst AUROC is obtained when both positive and negative dynamic class homophily levels are 0.5. More generally, the best AUROC across different positive and negative class dynamic homophily levels is achieved by different layers of a GCN. This result leads to potential insights in how to design more powerful dynamic GNNs that can adapt to both high and low dynamic homophily settings across time. Specifically, one way to overcome low dynamic homophily while maintaining high performance in high dynamic homophily settings is to leverage the intermediate representations of GCNs similar to designs in the static setting \citep{zhu2020beyond} since even layers can recover from low dynamic homophily, while odd layers may obtain the best AUROC in high dynamic homophily settings. This allows a GCN to obtain the best performance across different positive and negative class dynamic homophily levels as the dynamic graph changes across time.

\textbf{New Homophily Measure for Dynamic Graphs.} Our investigation of GCN discriminative performance in the dynamic setting suggests a new homophily measure for the dynamic setting. In the binary case, we define the overall dynamic homophily as the average of the positive and negative dynamic homophily levels. Defining dynamic homophily as an average correctly accounts for a GNNs ability to discriminate between each of the classes since if the average decreases the distance between classes also decreases. More generally, in the multiclass case, we use the following definition:

\begin{definition}[Dynamic homophily]
The dynamic homophily at timestep $t$ is defined as,
\al{h_t^D &= \frac{1}{\abs{C}}\sum_{c\in C} h_t^c, \text{ where} \\ h_t^c &= \mbb{P}(y_{t+1}(i) = y_{t}(j)\mid j \in \hat{\mc{N}}_t(i), y_{t+1}(i)=c).}
\end{definition}

In the dynamic multiclass setting, there is no single metric that fully captures the performance of a GNN similar to the static multiclass setting. In order to better capture performance in multiclass settings, in Appendix \ref{app:mult} we propose and analyze the dynamic compatibility matrix, an extension of the class compatibility in static settings~\cite{zhu2021graph,zhu2023heterophily}. Moreover, while our definition for dynamic homophily assumes a discrete dynamic graph, dynamic homophily can be easily extended to continuous dynamic graphs, and in Appendix \ref{app:cont} we propose a straightforward extension. 

\section{EXPERIMENTAL SETUP}

To demonstrate that our theoretical results hold in real-world dynamic graphs and for a variety of GNNs, we test if dynamic homophily correlates well with GNN performance across dynamic node classification tasks. We first consider tasks where the goal is to predict the spread of a signal since these tasks are fundamental problems arising in a variety of domains \citep{centola2007complex, guilbeault2018complex}. For example in public health, epidemiologists are interested in how a disease spreads, or in social network analysis, moderators aim to prevent the spread of misinformation \citep{leskovec2007dynamics, gomez2013modeling}. We next consider a biological task where the goal is to determine active proteins at each timestep, extending our experiments beyond signal spreading. Determining active proteins is an important task since it elucidates signaling pathways, information flow, and various protein functions \citep{przytycka2010toward, holme2012temporal}. We explore the relationship between dynamic homophily and GNN performance across time, considering both standard GNN designs and designs that aim to address low static homophily.

\subsection{Pseudo-synthetic Graph Datasets} 

To model the spread of infectious disease, we start by generating the dynamic graph structure. Given the structure, we generate features and labels based on an epidemiological model. In generating our structures, we utilize real-world dynamic graphs representing various types of social interactions: \textbf{UCI}, a social network at the University of California Irvine \citep{panzarasa2009patterns}, \textbf{Bitcoin}, a transaction network from a Bitcoin platform \citep{kumar2016edge}, and \textbf{Math}, a forum network on the website Math Overflow \citep{paranjape2017motifs}. In each network, structure changes over time such that $\mbf{A}_t \neq \mbf{A}_{t+1}$ for all $t$.

Given these real network structures, we synthetically generate labels and features using the susceptible-infected (SI) epidemiological model \citep{kermack1927contribution, newman2002spread}. Using the SI model, we independently generate 60 dynamic graphs for each structure. Each graph is generated by first sampling node infection parameters and statuses at $t=0$, then simulating the SI process. Here, node features include a node's infectivity and susceptibility parameters used in determining its future label and infection status at time $t$. With these dynamic graphs, we randomly split these data into equally sized sets of 20 graphs for training, validation, and testing. Thus, all nodes in the test set are unseen since the dynamic graphs are new.

\subsection{Real-world Dynamic Graph Datasets}
\textbf{Higgs Networks.} We utilize dynamic graphs from the Higgs dataset \citep{de2013anatomy}, a collection of dynamic social networks. The dataset monitors the spread of information about the discovery of the Higgs boson on Twitter before, during, and after its announcement. Here, a node's feature vector is a learnable node embedding concatenated with whether or not the information has yet reached the node. The task is to predict the times at which the information will reach each user. In our experiments, we consider four separate dynamic graphs each spanning 24 hours: \textbf{Higgs 1}, \textbf{Higgs 2}, \textbf{Higgs 3}, and \textbf{Higgs 4}, each representing different days of the Twitter network. For each dynamic graph we split the graph chronologically into 7 train graphs, 7 validation graphs, and 10 test graphs. In the Higgs networks, all nodes appear in the train, validation, and test sets. Each dynamic graph exhibits vastly different spreading behaviors since each day corresponds to a different period of the spreading process \citep{de2013anatomy}.

\begin{figure*}[h!]
     \centering
     \begin{subfigure}[b]{.325\textwidth}
         \centering
         \includegraphics[width=\textwidth]{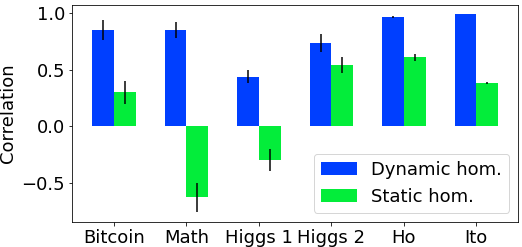}
         \caption{Correlations with GCN AUROC}
         \label{fig:sgc}
     \end{subfigure}
     \hfill
     \begin{subfigure}[b]{.325\textwidth}
         \centering
         \includegraphics[width=\textwidth]{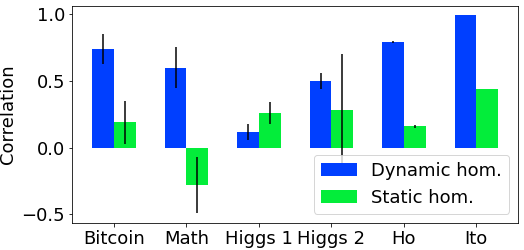}
         \caption{Correlations with GIN AUROC}
     \end{subfigure}
     \hfill
     \begin{subfigure}[b]{.325\textwidth}
         \centering
         \includegraphics[width=\textwidth]{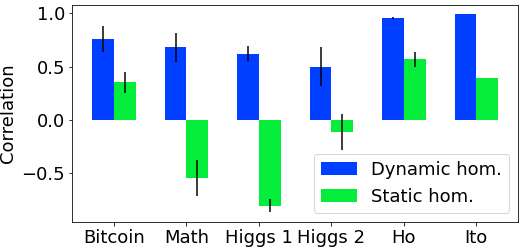}
         \caption{Correlations with GAT AUROC}
     \end{subfigure}
     \centering
     \begin{subfigure}[b]{.325\textwidth}
         \centering
         \includegraphics[width=\textwidth]{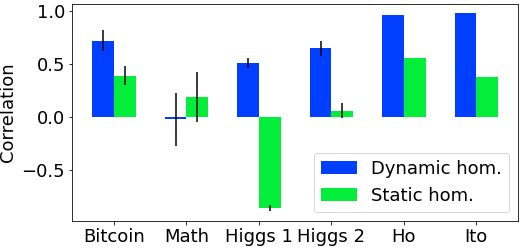}
         \caption{Correlations with SAGE AUROC}
         \label{fig:sgc}
     \end{subfigure}
     \hfill
     \begin{subfigure}[b]{.325\textwidth}
         \centering
         \includegraphics[width=\textwidth]{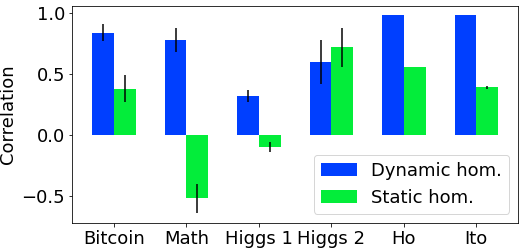}
         \caption{Correlations with GCNII AUROC}
     \end{subfigure}
     \hfill
     \begin{subfigure}[b]{.325\textwidth}
         \centering
         \includegraphics[width=\textwidth]{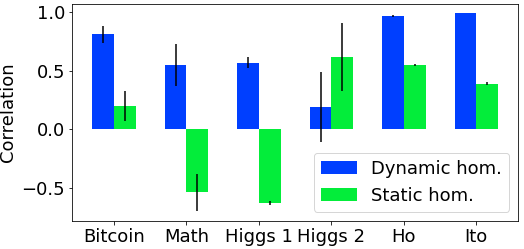}
         \caption{Correlations with FA-GCN AUROC}
     \end{subfigure}
    \caption{Mean $\pm$ standard deviations of Spearman's rank correlation coefficient between GNN AUROC and homophily measures across graphs in the test set for a subset of datasets. For most GNN and dynamic graph combinations, dynamic homophily has a higher correlation with GNN performance compared to static homophily.}
    \label{fig:3}
\end{figure*}

\textbf{Protein-Protein Interaction Networks.} We also utilize dynamic graphs from a biological repository of dynamic protein-protein interaction networks (DPPIN) \citep{fu2022dppin}. The datasets consist of dynamic protein-protein interactions of yeast cells where the task is to predict the times at which proteins are active. Here, node features include information describing proteins such as protein class and whether proteins are unknown or verified. Node features also include active status at time $t$. In our experiments, we consider four separate dynamic graphs: \textbf{Gavin}, \textbf{Ho}, \textbf{Ito}, and \textbf{Uetz} each spanning 36 timestamps of different protein interaction networks. We split the graph chronologically into 10 train graphs, 10 validation graphs, and 16 test graphs. Similar to the Higgs networks, all nodes appear in the train, validation, and test sets. The full details of all dynamic graphs are summarized in Appendix \ref{app:exps}\footnote{Code can be found at:\\https://github.com/MLD3/UnderstandingDynamicGraphs}.

\subsection{Training and Evaluation Details} 
\textbf{Models.} To explore the differences between homophilous and heterophilous GNNs, we train and evaluate many different GNNs: \textbf{SGC} \citep{wu2019simplifying}, \textbf{GCN} \citep{kipf2017semi}, \textbf{GIN} \citep{xu2018how}, and \textbf{GAT} \citep{velickovic2018graph} are homophilous baselines since they do not explicitly include designs to address heterophily, while \textbf{SAGE} \citep{hamilton2017inductive}, \textbf{GCNII} \citep{li2018deeper}, and \textbf{FA-GCN} \citep{bo2021beyond} are heterophilous since they adopt additional designs to improve performance in heterophilous settings \citep{LovelandZHFSK23discrepancies, zhu2023heterophily}. Given the construction of the dynamic graphs and existing results in \citet{fu2022dppin}, we do not expect dynamic GNNs that leverage the full temporal signal to perform better than static ones. Yet for completeness, in the Appendix we include results on the following dynamic GNNs: \textbf{DGNN} \citep{manessi2020dynamic, narayan2018learning, chen2022gc}, \textbf{GCRN} \citep{seo2018structured}, and \textbf{EvolveGCN} \citep{pareja2020evolvegcn}, where we find consistent results for the dynamic GNNs.

\noindent \textbf{Training and Evaluation.} During training, we apply a GNN to each static graph at time $t$ making predictions about the labels $\mbf{y}_{t+1}$ at the next timestep, minimizing the binary cross entropy loss. In the infectious disease and Higgs networks we only consider predictions for nodes that the signal has not reached. We evaluate each model on the same held out test set of graphs, measuring the AUROC. We then report the mean AUROC across time for each dynamic graph, and the mean and standard deviation of the AUROCs across all graphs in the test set. To measure to what extent dynamic homophily captures GNN discriminative power, we measure Spearman's rank correlation coefficient between dynamic homophily and GNN performance across time for each graph in the test set. We report the mean and standard deviation of the correlations across the graphs in the test set. For comparison, we also measure the correlation between static homophily with respect to snapshots of the dynamic graph and GNN performance. We provide training and evaluation details including computation of dynamic and static homophily, hyperparameter procedures, and reproducibility guidelines in Appendix \ref{app:exps}.

\begin{figure*}[h!]
     \centering
     \begin{subfigure}[b]{.325\textwidth}
         \centering
         \includegraphics[width=\textwidth]{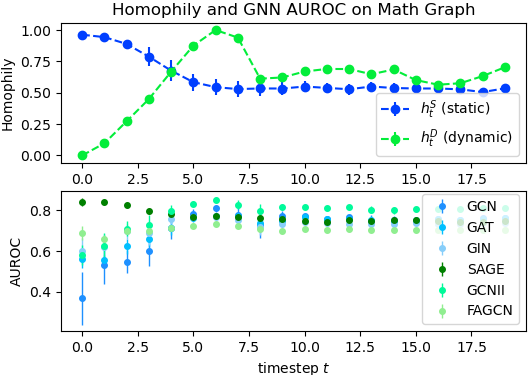}
         \caption{Math dynamic graph}
         \label{fig:kreg}
     \end{subfigure}
     \hfill
     \begin{subfigure}[b]{.325\textwidth}
         \centering
         \includegraphics[width=\textwidth]{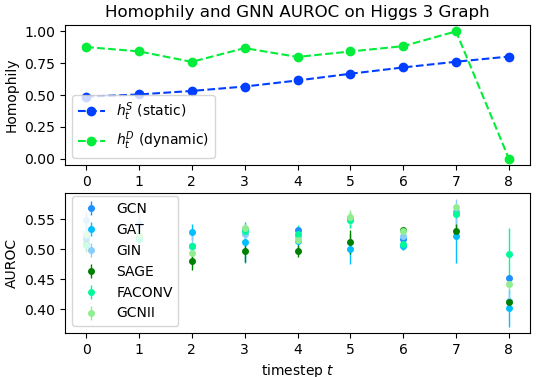}
         \caption{Higgs 3 dynamic graph}
     \end{subfigure}
     \hfill
     \begin{subfigure}[b]{.325\textwidth}
         \centering
         \includegraphics[width=\textwidth]{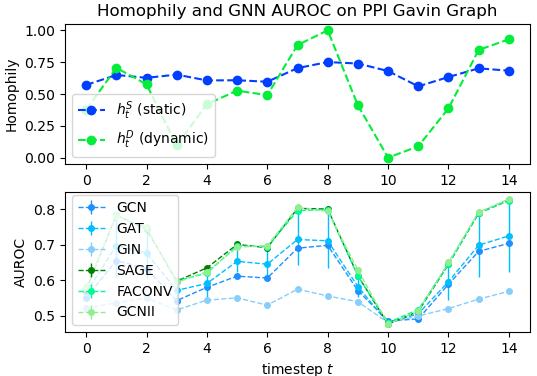}
         \caption{Gavin dynamic graph}
     \end{subfigure}
    \caption{Mean and standard deviations of dynamic homophily, static homophily, and AUROC across all graphs in the test set for the Math, Higgs 3, and Gavin dynamic graph. For all three dynamic graphs, static homophily stays high across time, while dynamic homophily and GNN performances exhibit the same trends. We do not connect AUROCs across signal spreading tasks since we only predict for nodes which the signal has not reached.}
    \label{fig:4}
\end{figure*}

\begin{table*}[h!]
\begingroup
\setlength{\tabcolsep}{6pt} 
\scriptsize
\centering
\caption{Mean ± standard deviation of GNN AUROC across time for each dynamic graph in the test set. In general, GNNs with heterophilous designs (SAGE, GCNII, FA-GCN) outperform GNNs with homophilous designs (SGC, GCN, GIN, GAT), suggesting that heterophilous design choices can help alleviate low dynamic homophily.}
\label{tab:3}
\resizebox{\textwidth}{!}{
\begin{tabular}{ l @{\qquad} c c c l c c c l c c c } 
\toprule
& \multicolumn{3}{c }{\textbf{Epidemiological Dynamic Graphs}} && \multicolumn{3}{c }{\textbf{Higgs Dynamic Graphs}} && \multicolumn{3}{c }{\textbf{Protein-protein Dynamic Graphs}} \\
\cmidrule{2-4} \cmidrule{6-8} \cmidrule{10-12}
 \textbf{GNN}  &  \textbf{UCI} & \textbf{Bitcoin} & \textbf{Math} && \textbf{Higgs 1} & \textbf{Higgs 2} & \textbf{Higgs 3} && \textbf{Gavin} & \textbf{Ho} & \textbf{Ito} \\ 
 \midrule
  \textbf{SGC} (\textit{hom.}) & 0.84 ± 0.01 & 0.85 ± 0.01 & 0.82 ± 0.01 && 0.54 ± 0.02 & 0.53 ± 0.01 & \grn{0.53 ± 0.00} && 0.60 ± 0.00 & 0.60 ± 0.00 & 0.67 ± 0.00 \\
 \textbf{GCN} (\textit{hom.}) & 0.84 ± 0.01 & 0.85 ± 0.01 & 0.82 ± 0.01 && 0.56 ± 0.01 & 0.54 ± 0.01 & 0.52 ± 0.01 && 0.61 ± 0.00 & 0.60 ± 0.00 & 0.67 ± 0.00 \\
\textbf{GIN} (\textit{hom.}) & 0.85 ± 0.01 & \grn{0.88 ± 0.01} & 0.86 ± 0.01 && 0.49 ± 0.00 & 0.52 ± 0.01 & 0.52 ± 0.00 && 0.54 ± 0.00 & 0.53 ± 0.00 & 0.63 ± 0.00 \\
 \textbf{GAT} (\textit{hom.}) & 0.79 ± 0.01 & 0.84 ± 0.01 & 0.81 ± 0.01 && 0.60 ± 0.02 & 0.51 ± 0.01 & 0.50 ± 0.01 && 0.63 ± 0.03 & 0.60 ± 0.01 & \grn{0.68 ± 0.00} \\
 \midrule
 \textbf{SAGE} (\textit{het.}) & \grn{0.86 ± 0.01} & 0.87 ± 0.01 & \grn{0.88 ± 0.01} && \grn{0.61 ± 0.00} & \grn{0.55 ± 0.00} & 0.50 ± 0.01 && \grn{0.68 ± 0.00} & 0.68 ± 0.00 & \grn{0.68 ± 0.00} \\
  \textbf{GCNII} (\textit{het.}) & 0.84 ± 0.01 & \grn{0.88 ± 0.01} & \grn{0.88 ± 0.01} && 0.56 ± 0.01 & 0.53 ± 0.00 & 0.52 ± 0.01 && \grn{0.68 ± 0.00} & \grn{0.69 ± 0.00} & \grn{0.68 ± 0.00} \\
 \textbf{FA-GCN} (\textit{het.}) & 0.79 ± 0.01 & 0.83 ± 0.01 & 0.80 ± 0.01 && 0.57 ± 0.01 & 0.54 ± 0.01 & 0.52 ± 0.00 && \grn{0.68 ± 0.00} & \grn{0.69 ± 0.00} & \grn{0.68 ± 0.00} \\
 \bottomrule
\end{tabular}
}
\endgroup
\end{table*}

\section{EXPERIMENTAL RESULTS}
We empirically compare a variety of GNNs based on their performance as dynamic homophily changes. We aim to answer the following research questions: 
\begin{itemize}
    \item \textbf{RQ1}: To what extent do dynamic homophily and static homophily, based on snapshots in time, correlate with GNN discriminative performance in general dynamic node classification tasks where our assumptions made in our analyses are violated?
         
    \item \textbf{RQ2}: Do the observed trends change when considering GNNs specifically designed to perform well in settings with low static homophily? And how do such GNNs perform in settings with low dynamic homophily?
\end{itemize}

\subsection{(RQ1) How does static and dynamic homophily correlate with GNN AUROC?}
In Figure \ref{fig:3}, we compare average correlations obtained by dynamic and static homophily with GNN performance for a representative subset of GNN and dynamic graph combinations. We observe similar trends on the set of full results and present them in the Appendix \ref{sec:rem}. Across 43 out of 44 homophilous GNN and dynamic graph combinations, the average correlation between dynamic homophily and GNN performance exceeds the correlation between static homophily and GNN performance. Across the 44 combinations, dynamic homophily achieves a median correlation of 0.75 interquartile range (IQR): (0.59, 0.96) which is significantly greater than the median correlation achieved by static homophily of 0.26 IQR: (-0.15, 0.47). 

Data in Figure \ref{fig:3} are averaged across graphs in the test set, but also across time. To fully capture changes to GNN performance, we measure AUROC, dynamic homophily, and static homophily at each timestep for the Math, Higgs 3, and Gavin graphs (Figure \ref{fig:4}). While dynamic homophily exhibits the same trends as homophilous GNN performance, static homophily does not since it stays high across all three graphs for all timesteps. We observe consistent results on the remaining graphs and include them in Appendix \ref{app:plots}.

\subsection{(RQ2) How do heterophilous GNNs perform in low dynamic homophily?}

Similar to trends across homophilous GNNs, we find that dynamic homophily also correlates well with heterophilous GNN performance, while static homophily does not. Across the 33 combinations of heterophilous GNNs and dynamic graphs, dynamic homophily achieves a median correlation of 0.78 IQR: (0.55, 0.97) which is significantly greater than the median correlation achieved by static homophily of 0.38 IQR: (0.04, 0.54). GNNs with heterophilous designs also tend to perform higher on average compared to GNNs with homophilous designs (Table \ref{tab:3}). In Figure \ref{fig:4}, we compare the performance across time between the two design types. When dynamic homophily is high, the task is generally easy, and we observe strong discriminative performance from both homophilous and heterophilous GNNs. However, when dynamic homophily is low, the task becomes much more difficult, and performance gaps emerge between the two design types in favor of the heterophilous GNNs. The results suggest that heterophilous designs in the static setting could also alleviate low dynamic homophily in the dynamic setting.

Although dynamic homophily generally correlates well with the discriminative performance for GNNs of different designs and dynamic graph datasets, there are specific combinations of GNNs and dynamic graphs in which dynamic homophily is only weakly correlated with GNN performance. Specifically, the performance of \textbf{SAGE} does not correlate with dynamic homophily for the Math graph. This may be due to several limitations. First, our theoretical analysis relies on the GCN aggregation. Thus, dynamic homophily is not guaranteed to correlate well with GNNs that leverage more complex aggregations. Second, as shown in Theorems \ref{thm:2} and \ref{thm:3}, dynamic homophily may not reflect the distance in node representations at a particular timestep if there are too few nodes in both classes at that timestep.



\section{DISCUSSION AND CONCLUSION}

In the context of node classification on dynamic graphs, we present the first theoretical and empirical analysis of GNNs and homophily. Specifically, we analyze GCNs and characterize their discriminative performance in dynamic settings. Based on our analysis, we propose dynamic homophily, a new homophily measure that characterizes GCN discriminative power in a dynamic setting. 

Our work builds on previous work that has focused on GNN performance and homophily in static settings. More specifically, \citet{ma2022is}, \citet{yan2022two}, \citet{li2022finding} and \citet{Zhu0IKF24} all measured GCN node representations and their relationship to node-level homophily in static graphs. Their analyses focus on distance between classes in order to improve GNN performance in static settings. We also look at the node representations, but in addition to studying the distance between classes, we analyze the distance within classes which provides additional insights into designing better GNNs for dynamic settings. 

Our analysis suggests that leveraging intermediate GNN representations can mitigate low dynamic homophily and allow a GNN to adapt to the dynamic graph as it changes across time. In our empirical analyses, we demonstrate that dynamic homophily is highly correlated with the performance of many variations of GNNs across a variety of dynamic node classification tasks from epidemiology, social network analysis, and molecular biology. More broadly, our results indicate that popular homophilous and heterophilous GNNs are not robust to low dynamic homophily, and as a result new approaches may be warranted. Going forward, our results have the potential to inform future GNN designs. Moreover, our work is an important building block for future work that can analyze more complex dynamic settings such as homophily in the context of local performance discrepancies extending the global view of dynamic homophily and general dynamic graph tasks beyond node classification. Overall, our work represents a step toward understanding the discriminative power of GNNs in the dynamic setting.

\subsubsection*{Acknowledgements}

This material is based on work supported by the U.S. Department of Energy, Office of Science, Office of Advanced Scientific Computing Research, Department of Energy Computational Science Graduate Fellowship under Award Number DE-SC0023112. It was also partially supported by National Science Foundation under Grant No. IIS~2212143. We thank the anonymous reviewers and members of the MLD3 lab for their valuable feedback.


\section*{Checklist}



 \begin{enumerate}

 \item For all models and algorithms presented, check if you include:
 \begin{enumerate}
   \item A clear description of the mathematical setting, assumptions, algorithm, and/or model. \textbf{Yes}
   \item An analysis of the properties and complexity (time, space, sample size) of any algorithm. \textbf{Yes}
   \item (Optional) Anonymized source code, with specification of all dependencies, including external libraries. \textbf{Yes}.
 \end{enumerate}

 \item For any theoretical claim, check if you include:
 \begin{enumerate}
   \item Statements of the full set of assumptions of all theoretical results. \textbf{Yes}
   \item Complete proofs of all theoretical results. \textbf{Yes}
   \item Clear explanations of any assumptions. \textbf{Yes}
 \end{enumerate}

 \item For all figures and tables that present empirical results, check if you include:
 \begin{enumerate}
   \item The code, data, and instructions needed to reproduce the main experimental results (either in the supplemental material or as a URL). \textbf{No}
   \item All the training details (e.g., data splits, hyperparameters, how they were chosen). \textbf{Yes}
         \item A clear definition of the specific measure or statistics and error bars (e.g., with respect to the random seed after running experiments multiple times). \textbf{Yes}
         \item A description of the computing infrastructure used. (e.g., type of GPUs, internal cluster, or cloud provider). \textbf{Yes}
 \end{enumerate}

 \item If you are using existing assets (e.g., code, data, models) or curating/releasing new assets, check if you include:
 \begin{enumerate}
   \item Citations of the creator If your work uses existing assets. \textbf{Yes}
   \item The license information of the assets, if applicable. \textbf{Not Applicable}
   \item New assets either in the supplemental material or as a URL, if applicable. \textbf{Not Applicable}
   \item Information about consent from data providers/curators. \textbf{Not Applicable}
   \item Discussion of sensible content if applicable, e.g., personally identifiable information or offensive content. \textbf{Not Applicable}
 \end{enumerate}

 \item If you used crowdsourcing or conducted research with human subjects, check if you include:
 \begin{enumerate}
   \item The full text of instructions given to participants and screenshots. \textbf{Not Applicable}
   \item Descriptions of potential participant risks, with links to Institutional Review Board (IRB) approvals if applicable. \textbf{Not Applicable}
   \item The estimated hourly wage paid to participants and the total amount spent on participant compensation. \textbf{Not Applicable}
 \end{enumerate}

 \end{enumerate}

\appendix

\onecolumn
\aistatstitle{Supplementary Material}


\tableofcontents

\newpage

\section{REMAINING EXPERIMENTAL RESULTS \label{sec:rem}}

In Figure \ref{fig:5}, we present the average correlations obtained by dynamic and static homophily with GNN performance on the remaining GNN and dataset combinations. We observe consistent trends from the main paper and for most combinations dynamic homophily obtains a higher correlation than static homophily. In Table \ref{tab:2}, we compare the average performance obtained by all GNNs on all dynamic graph datasets. We observe consistent trends from the main paper and generally GNNs with heterophilous designs outperform GNNs with homophilous designs. 

\begin{figure*}[h!]
     \centering
     \begin{subfigure}[b]{.325\textwidth}
         \centering
         \includegraphics[width=\textwidth]{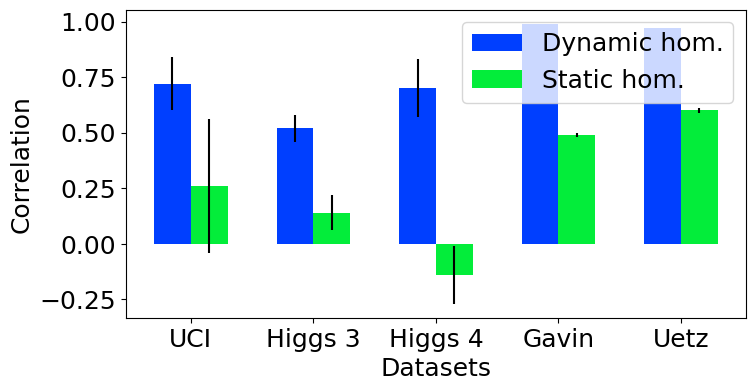}
         \caption{Correlations with GCN AUROC}
         \label{fig:sgc}
     \end{subfigure}
     \hfill
     \begin{subfigure}[b]{.325\textwidth}
         \centering
         \includegraphics[width=\textwidth]{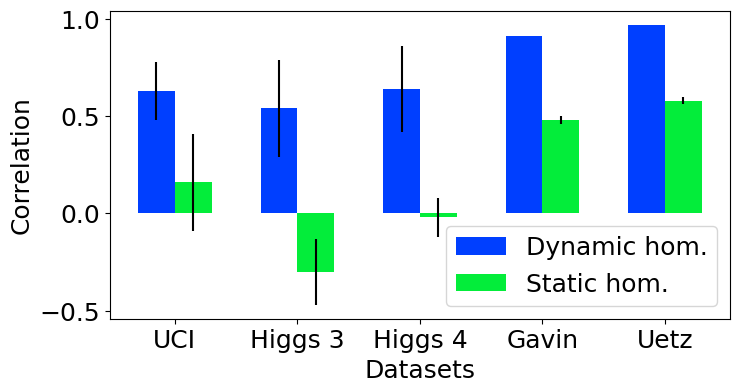}
         \caption{Correlations with GIN AUROC}
     \end{subfigure}
     \hfill
     \begin{subfigure}[b]{.325\textwidth}
         \centering
         \includegraphics[width=\textwidth]{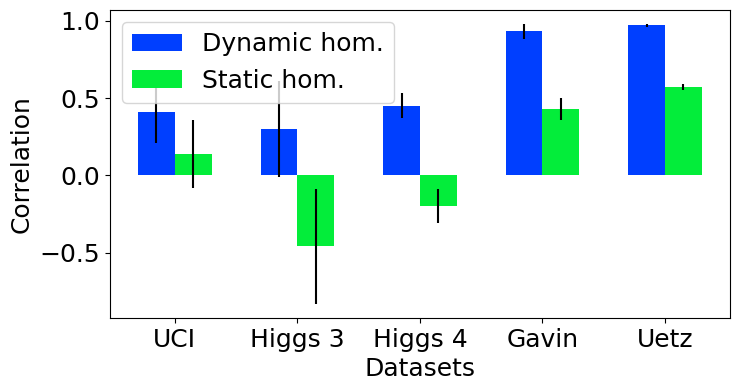}
         \caption{Correlations with GAT AUROC}
     \end{subfigure}
     \centering
     \begin{subfigure}[b]{.325\textwidth}
         \centering
         \includegraphics[width=\textwidth]{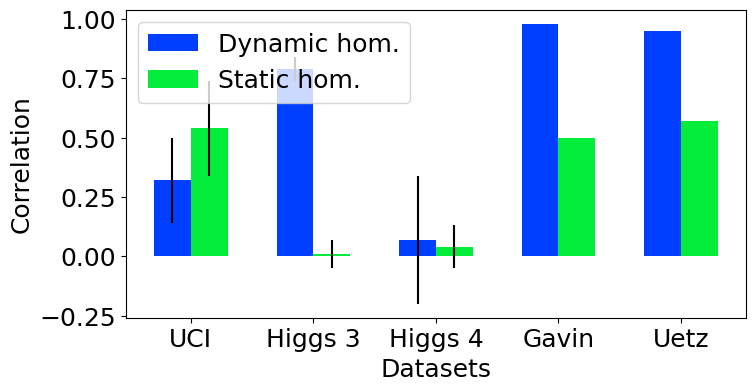}
         \caption{Correlations with SAGE AUROC}
         \label{fig:sgc}
     \end{subfigure}
     \hfill
     \begin{subfigure}[b]{.325\textwidth}
         \centering
         \includegraphics[width=\textwidth]{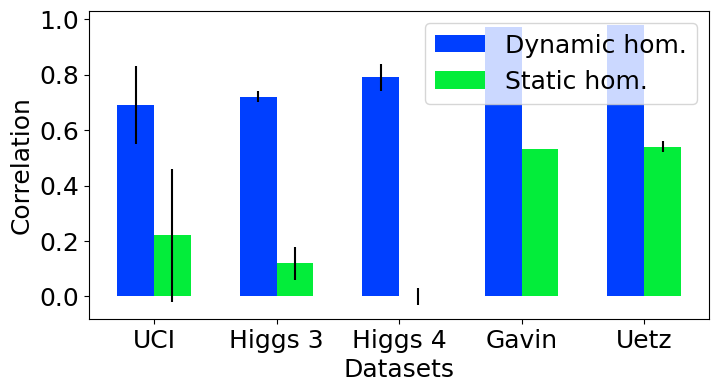}
         \caption{Correlations with GCNII AUROC}
     \end{subfigure}
     \hfill
     \begin{subfigure}[b]{.325\textwidth}
         \centering
         \includegraphics[width=\textwidth]{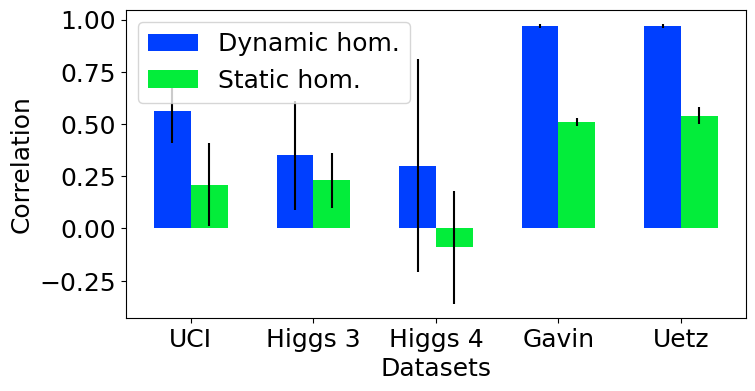}
         \caption{Correlations with FA-GCN AUROC}
     \end{subfigure}
     \begin{subfigure}[b]{.325\textwidth}
         \centering
         \includegraphics[width=\textwidth]{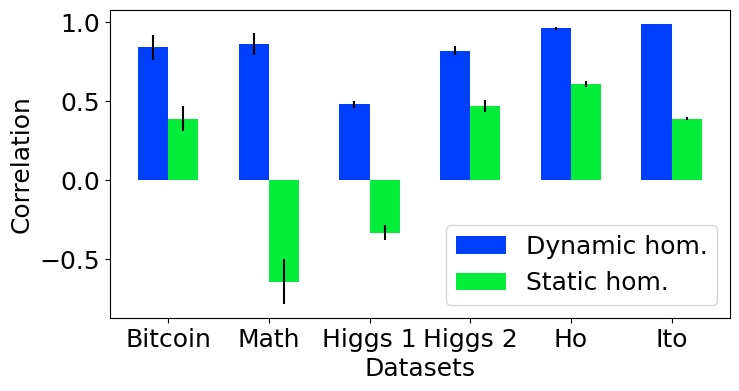}
         \caption{Correlations with SGC AUROC}
     \end{subfigure}
     \begin{subfigure}[b]{.325\textwidth}
         \centering
         \includegraphics[width=\textwidth]{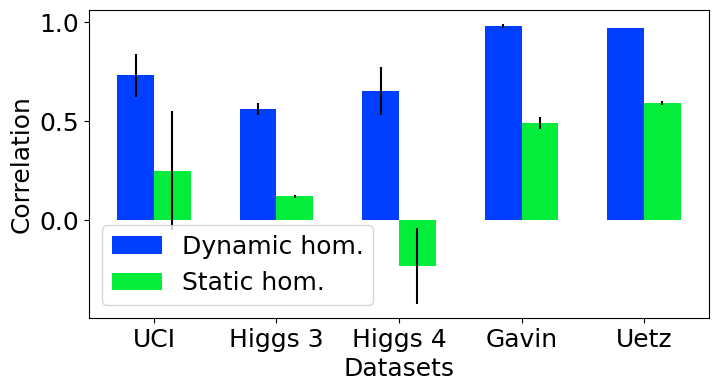}
         \caption{Correlations with SGC AUROC}
     \end{subfigure}
    \caption{\footnotesize Mean $\pm$ standard deviations of Spearman's rank correlation coefficient between GNN AUROC and homophily measures across graphs in the test set for the remaining dynamic graph datasets. For most GNN and dynamic graph combinations, dynamic homophily has a higher correlation with GNN performance compared to static homophily.}
    \label{fig:5}
\end{figure*}

\begin{table*}[h!]
\begingroup
\setlength{\tabcolsep}{6pt} 
\scriptsize
\centering
\caption{Mean ± standard deviation of GNN AUROC across time for each dynamic graph in the test set. In general, GNNs with heterophilous designs (SAGE, GCNII, FA-GCN) outperform GNNs with homophilous designs (SGC, GCN, GIN, GAT), suggesting that heterophilous design choices can help alleviate low dynamic homophily in the dynamic setting.}
\label{tab:2}
\resizebox{\textwidth}{!}{
\begin{tabular}{ l @{\qquad} c c c l c c c c l c c c c } 
\toprule
& \multicolumn{3}{c }{\textbf{Epidemiological Dynamic Graphs}} && \multicolumn{4}{c }{\textbf{Higgs Dynamic Graphs}} && \multicolumn{4}{c }{\textbf{Protein-protein Dynamic Graphs}} \\
\cmidrule{2-4} \cmidrule{6-9} \cmidrule{11-14}
 \textbf{GNN}  &  \textbf{UCI} & \textbf{Bitcoin} & \textbf{Math} && \textbf{Higgs 1} & \textbf{Higgs 2} & \textbf{Higgs 3} & \textbf{Higgs 4} && \textbf{Gavin} & \textbf{Ho} & \textbf{Ito} & \textbf{Uetz}\\ 
 \midrule
  \textbf{SGC} (\textit{hom.}) & 0.84 ± 0.01 & 0.85 ± 0.01 & 0.82 ± 0.01 && 0.54 ± 0.02 & 0.53 ± 0.01 & \grn{0.53 ± 0.00} & 0.53 ± 0.01 && 0.60 ± 0.00 & 0.60 ± 0.00 & 0.67 ± 0.00 & 0.65 ± 0.00 \\
 \textbf{GCN} (\textit{hom.}) & 0.84 ± 0.01 & 0.85 ± 0.01 & 0.82 ± 0.01 && 0.56 ± 0.01 & 0.54 ± 0.01 & 0.52 ± 0.01 & \grn{0.53 ± 0.00} && 0.61 ± 0.00 & 0.60 ± 0.00 & 0.67 ± 0.00 & 0.65 ± 0.00 \\
\textbf{GIN} (\textit{hom.}) & 0.85 ± 0.01 & \grn{0.88 ± 0.01} & 0.86 ± 0.01 && 0.49 ± 0.00 & 0.52 ± 0.01 & 0.52 ± 0.00 & 0.52 ± 0.00 && 0.54 ± 0.00 & 0.53 ± 0.00 & 0.63 ± 0.00 & 0.64 ± 0.00 \\
 \textbf{GAT} (\textit{hom.}) & 0.79 ± 0.01 & 0.84 ± 0.01 & 0.81 ± 0.01 && 0.60 ± 0.02 & 0.51 ± 0.01 & 0.50 ± 0.01 & 0.50 ± 0.01 && 0.63 ± 0.03 & 0.60 ± 0.01 & \grn{0.68 ± 0.00} & \grn{0.67 ± 0.00} \\
 \midrule
 \textbf{SAGE} (\textit{het.}) & \grn{0.86 ± 0.01} & 0.87 ± 0.01 & \grn{0.88 ± 0.01} && \grn{0.61 ± 0.00} & \grn{0.55 ± 0.00} & 0.50 ± 0.01 & 0.49 ± 0.00 && \grn{0.68 ± 0.00} & 0.68 ± 0.00 & \grn{0.68 ± 0.00} & \grn{0.67 ± 0.00} \\
  \textbf{GCNII} (\textit{het.}) & 0.84 ± 0.01 & \grn{0.88 ± 0.01} & \grn{0.88 ± 0.01} && 0.56 ± 0.01 & 0.53 ± 0.00 & 0.52 ± 0.01 & \grn{0.53 ± 0.00} && \grn{0.68 ± 0.00} & \grn{0.69 ± 0.00} & \grn{0.68 ± 0.00} & \grn{0.67 ± 0.00} \\
 \textbf{FA-GCN} (\textit{het.}) & 0.79 ± 0.01 & 0.83 ± 0.01 & 0.80 ± 0.01 && 0.57 ± 0.01 & 0.54 ± 0.01 & 0.52 ± 0.00 & 0.52 ± 0.00 && \grn{0.68 ± 0.00} & \grn{0.69 ± 0.00} & \grn{0.68 ± 0.00} & \grn{0.67 ± 0.00} \\
 \bottomrule
\end{tabular}
}
\endgroup
\end{table*}

\vfill

\section{PROOFS OF MAIN RESULTS}

\subsection{Characterization of AUROC after $l$ GCN Layers \label{pf:lem1}}

\begin{lemma}
The expected AUROC of $f^{(l)}$ at timestep $t$ can be written as follows, 
\eq{\mbb{E}[A_t(f^{(l)})] = 1 - \Phi\prn{-\frac{\mbb{E}_{i | i \in V_{t+1}^+}\brt{\mbf{h}^{(l)}_{t}(i)} -\mbb{E}_{j | j \in V_{t+1}^-}\brt{\mbf{h}^{(l)}_{t}(j)}}{\mbb{V}_{i, j| i \in V_{t+1}^+, j \in V_{t+1}^-}[\mbf{h}^{(l)}_{t}(i)  + \mbf{h}^{(l)}_{t}(j)]}}}
where $\Phi$ is the cumulative distribution function of the Gaussian distribution, and $V_{t+1}^+$ and $V_{t+1}^-$ are the nodes in the positive and negative classes at time $t+1$, respectively.
\label{lem:2}
\end{lemma}

\begin{proof}
The expected AUROC of $f^{(l)}$ at timestep $t$ is equivalent to the probability that $f^{(l)}$ ranks a random positive node higher than a random negative node. Thus, the expected AUROC of $f^{(l)}$ at timestep $t$ can also be expressed as,
\al{\mbb{E}[A_t(f^{(l)})] &= \mbb{P}(f^{(l)}(i) > f^{(l)}(j) \mid i \in V_{t+1}^+, j \in V_{t+1}^-) \\
&= \mbb{P}(f^{(l)}(i) - f^{(l)}(j) > 0 \mid i \in V_{t+1}^+, j \in V_{t+1}^-)}

Defining random variable $Z_t$ as the linear combination $Z_t = f^{(l)}(i) - f^{(l)}(j)$ given a random node $i \in V_{t+1}^+$ and random node $j \in V_{t+1}^-$, the expected AUROC amounts to the quantity: 
\eq{\mbb{E}[A_t(f^{(l)})] = 1 - \Phi_{Z_t}(0)}

where $\Phi$ is the cumulative distribution function of the random variable $Z_t$. Furthermore, since $Z_t$ is Gaussian distributed, we can apply the definition of the CDF for Gaussian distributions and express the expected AUROC at timestep $t$ as the following: 
\eq{\mbb{E}[A_t(f^{(l)})] = 1 - \Phi_{Z_t}(0) = 1 - \Phi\prn{-\frac{\mbb{E}[Z_t] }{\mbb{V}[Z_t]}} \label{eq:20}}
Substituting the definition of $Z_t$ into Equation \ref{eq:20} completes the proof.
\end{proof}

\subsection{Expected Distance for $l$ layer GNNs \label{pf:dist}}

\begin{theorem}
At time $t$ the difference in expected node representations between a future positive and negative node after $l$ layers of a GCN can be expressed as: 
\eq{\mbb{E}_{i | i \in V_{t+1}^+}\brt{\mbf{h}^{(l)}_{t}(i)} -\mbb{E}_{j | j \in V_{t+1}^-}\brt{\mbf{h}^{(l)}_{t}(j)} = 2 \cdot \mu_t \cdot (h_t^+ + h_t^- - 1)^l.}

where $\mu_t$ is the magnitude of the mean for all $\mbf{x}_t(i)$, $h_{t}^+$ is the probability node $i$'s label at time $t+1$ is the same as its neighbor's label at time $t$ given node $i$'s label is positive at $t+1$ such that: 
\eq{h_{t}^+ = \mbb{P}(y_{t+1}(i) = y_{t}(j)\mid j \in \hat{\mc{N}}_t(i), y_{t+1}(i) = +1),}

and $h_{t}^-$ is the probability node $i$'s label at time $t+1$ is the same as its neighbor's label at time $t$ given node $i$'s label is negative at $t+1$ such that: 
\eq{h_{t}^- = \mbb{P}(y_{t+1}(i) = y_{t}(j)\mid j \in \hat{\mc{N}}_t(i), y_{t+1}(i)=-1).}

We denote $h_{t}^+$ and $h_{t}^-$ as the positive and negative dynamic homophily levels, respectively.
\end{theorem}

As a proof sketch for Theorem \ref{thm:1}, we first show that the expected node representations after 1 GCN layer is a scaling of the expected initial node representation, where the scaling factor is an function of the positive and negative dynamic homophily levels. As a result, the difference in the expected node representations between a randomly sampled positive node and a randomly sampled negative node after 1 layer of a GCN is a function of dynamic homophily. This proves the theorem for $l=1$. We prove the result for the general $l$ layer case using induction. 

\begin{proof} The first step in proving Theorem \ref{thm:1} is to measure the expected node representation after one layer of a GNN with mean aggregation, starting with the negative case in which $i \in V_{t+1}^-$. Formally, we express this quantity as $\mbb{E}_{i \mid i \in V_{t+1}^-}\brt{\mbf{h}^{(1)}_{t}(i)}$, where the expectation is taken over all nodes $i \in V_{t+1}$ for which $y_{t+1}(i) = -1$. The first step in measuring this expectation is to use the propagation rule for a GNN with mean aggregation:
\eq{\mbb{E}_{i | i \in V_{t+1}^-}\brt{\mbf{h}^{(1)}_{t}(i)} =\sum_{ j \in \hat{\mathcal{N}}_t(i) } \mbb{E}_{i, j | i \in V_{t+1}^-} \brt{\frac{\mbf{x}_t(j)}{d_t(i) + 1}} \label{eq:1} }

The above tell us that the expected node representation is a sum over the neighbors of node $i$. Now, the key step in expressing the expected node representation is to \textit{partition} the neighbors of $i$ into those whose label at time $t$ is the same as or opposes $i$'s label at timestep $t +1$: 
\ml{\mbb{E}_{i | i \in V_{t+1}^-}\brt{\mbf{h}^{(1)}_{t}(i)}=\Big(\mbb{E}_{i, j | i\in V_{t+1}^-, j \in V_t^-} \brt{\mbf{x}_t(j)} \cdot \mbb{P}(y_{t+1}(i) = y_t(j) \mid j \in \hat{\mc{N}}_t(i), i \in V_{t+1}^-) \Big) \\
+ \Big(\mbb{E}_{i, j | i \in V_{t+1}^-, j \in V_t^+} \brt{\mbf{x}_t(j)} \cdot (1 - \mbb{P}(y_{t+1}(i) = y_t(j) \mid j \in \hat{\mc{N}}_t(i), i \in V_{t+1}^-)) \Big) \label{eq:16}}

In fact, this partition over the set of neighbors is determined precisely by negative dynamic homophily. Applying the definition of negative dynamic homophily to \eqref{eq:16} produces the following:
\eq{\mbb{E}_{i | i \in V_{t+1}^-}\brt{\mbf{h}^{(1)}_{t}(i)} = \mbb{E}_{i, j \mid i\in V_{t+1}^-, j \in V_t^-} \brt{\mbf{x}_t(j)} \cdot h_t^- + \mbb{E}_{i, j \mid i \in V_{t+1}^-, j \in V_t^+} \brt{\mbf{x}_t(j)} \cdot (1 - h_t^-)}

Applying our assumption that for all $t \in [0, T]$ and for all $i$, $y_t(i) \in \{-1, +1\}$ and $\mbf{x}_t(i) \sim N(y_t(i) \cdot \mu_t, \sigma^2)$ produces:
\eq{\mbb{E}_{i | i \in V_{t+1}^-}\brt{\mbf{h}^{(1)}_{t}(i)} = \mbb{E}_{j \mid j \in V_t^+} \brt{\mbf{x}_t(j)} \cdot -h_t^- + \mbb{E}_{j \mid j \in V_t^+} \brt{\mbf{x}_t(j)} \cdot (1 - h_t^-) = (1 - 2 h_t^-) \cdot \mu_t \label{eq:18}}

Following a similar derivation, we measure the expected node representation for the positive case in which $i \in V_{t+1}^+$: 
\eq{\mbb{E}_{i | i \in V_{t+1}^+}\brt{\mbf{h}^{(1)}_{t}(i)} = \mbb{E}_{j \mid j \in V_t^+} \brt{\mbf{x}_t(j)} \cdot h_t^+ + \mbb{E}_{j \mid j \in V_t^+} \brt{\mbf{x}_t(j)} \cdot (h_t^+ - 1) \\= (2h_t^+ - 1) \cdot \mu_t \label{eq:19}}

Subtracting \eqref{eq:18} from \eqref{eq:19}, yields our desired result, in which the distance between the expected node representations of the opposing classes becomes a function of the average dynamic homophily. We have proven the base case where $l=1$. In order to prove the induction step, we assume Theorem \ref{thm:1} holds at layer $l$, and show it holds at layer $l+1$. At layer $l + 1$ the key to expressing the difference in expected node representations is to do so in a recursive manner as follows:

\ml{\mbb{E}_{i | i \in V_{t+1}^+}\brt{\mbf{h}^{(l+1)}_{t}(i)} -\mbb{E}_{i | i \in V_{t+1}^-}\brt{\mbf{h}^{(l+1)}_{t}(i)} \\
= \Big(\mbb{E}_{j \mid j \in V_t^-} \brt{\mbf{h}^{(l)}_t(j)} \cdot h_t^+ + \mbb{E}_{j \mid j \in V_t^+} \brt{\mbf{h}^{(l)}_t(j)} \cdot (1 - h_t^+) \Big)  \\
- \Big(\mbb{E}_{j \mid j \in V_t^-} \brt{\mbf{h}^{(l)}_t(j)} \cdot h_t^- + \mbb{E}_{j \mid j \in V_t^+} \brt{\mbf{h}^{(l)}_t(j)} \cdot (1 - h_t^-) \Big) }

Applying our assumption that for all $t \in [0, T]$ and for all $i$, $y_t(i) \in \{-1, +1\}$, and $\mbf{x}_t(i) \sim \mc{N}(y_t(i) \cdot \mu_t, \sigma^2)$ and grouping like terms produces: 
\eq{\mbb{E}_{i |i \in V_{t+1}^+}\brt{\mbf{h}^{(l+1)}_{t}(i)} -\mbb{E}_{i | i \in V_{t+1}^-}\brt{\mbf{h}^{(l+1)}_{t}(i)} =\Big(\mbb{E}_{i \mid i \in V_t^+} \brt{\mbf{h}^{(l)}_t(i)} - \mbb{E}_{i \mid i \in V_t^-} \brt{\mbf{h}^{(l)}_t(i)} \Big) \cdot (2h_t^D - 1).}

Applying the inductive hypothesis completes the proof. 
\end{proof}

\subsection{Concentration of the Expected Distance \label{pf:con}}

\begin{theorem}
For any $\epsilon > 0$, the probability that at time $t$ the distance between the empirical and expected difference after $l$ GCN layers is larger than $\epsilon$ is bounded as follows:
{\eq{\mbb{P}\prn{\Big\lvert (\mu_{V_{t+1}^+}^{(l)} - \mu_{V_{t+1}^-}^{(l)}) - (\mbb{E}_{i |i \in V_{t+1}^+}[\mbf{h}^{(l)}_{t}(i)] -\mbb{E}_{i | i \in V_{t+1}^-}[\mbf{h}^{(l)}_{t}(i)]) \Big\rvert \geq \epsilon}  \leq \mc{O}\prn{e^{-\epsilon^2 L_{t, +}^{(l)}} + e^{-\epsilon^2 L_{t,-}^{(l)}}}, }}

where $\mu_{V_{t+1}^+}^{(l)}$ and $\mu_{V_{t+1}^-}^{(l)}$ are the empirical mean representations after $l$ GCN layers over positive and negative nodes at time $t+1$ respectively, and
\eq{L_{t, +}^{(l)} = \frac{ \abs{V_{t+1}^+}^2}{\sigma^4 \cdot \sum_{i \in V_{t+1}^+} \prn{\sum_{j \in \hat{\mc{N}}_t(i)} \frac{l}{d_t(j)^l}}^{2}}, \quad L_{t, -}^{(l)} = \frac{ \abs{V_{t+1}^-}^2}{\sigma^4 \cdot \sum_{i \in V_{t+1}^-} \prn{\sum_{j \in \hat{\mc{N}}_t(i)} \frac{l}{d_t(j)^l}}^{2}}.}
\end{theorem}

Before proving Theorem \ref{thm:2}, we first state the following relevant definitions and lemmas pertaining to the concentration of Gaussian random variables and Lipschitz functions.

\begin{definition}[Lipschitz functions]
Let $(X, d_x)$ and $(Y, d_Y)$ be metric spaces. A function $f: X \to Y$ is called Lipschitz if there exists $L \in \mbb{R}$ such that: 
\eq{d_Y(f(u), f(v)) \leq L \cdot d_X(u, v) \quad \text{for every } u, v \in X}

The infimum of all $L$ is called the Lipschitz norm of $f$ and is denoted $\lvert\lvert f \rvert\rvert_\text{Lip}$.
\end{definition}

\begin{lemma}[\citet{wainwright2019high}] 
Suppose $f: \mbb{R}^n \to \mbb{R}$ is $L$-Lipschitz with respect to Euclidean distance, and let $\mbf{x} = (x_0, \ldots, x_n) \sim \mc{N}(0, 1)$. Then for all $\epsilon \in \mbb{R}$, 
\eq{\mbb{P}(\abs{f(\mbf{x}) - \mbb{E}[f(x)]} \leq \epsilon) \leq 2\text{exp}\prn{\frac{-\epsilon^2}{2L^2}}}
\label{lem:lip_1}
\end{lemma}

\begin{lemma}[\citet{vershynin2018high}]
Every differentiable function $f: \mbb{R}^n \to \mbb{R}$ is Lipschitz, and 
\eq{\lvert\lvert f\rvert\rvert_\text{Lip} \leq \textrm{sup}_{x\in \mbb{R}^n} \lvert\lvert \nabla f(x) \rvert\rvert_2}
\label{lem:lip_2}
\end{lemma}

\begin{proof}
We aim to bound the difference between the expected difference in node representation and their empirical difference. To prove the bound, we use a similar derivation as in Theorem \ref{thm:2}. We first decompose the difference terms and treat each separately as follows:
\al{&\Big\lvert\Big(\frac{1}{\abs{V_{t+1}^+}}\sum_{i \in V_{t+1}^+}\mbf{h}_t^{(l)}(i) - \frac{1}{\abs{V_{t+1}^-}}\sum_{i \in V_{t+1}^-} \mbf{h}_t^{(l)}(i)\Big) - \Big(\mbb{E}_{i | i \in V_{t+1}^+}\brt{\mbf{h}^{(l)}_{t}(i)} -\mbb{E}_{i | i \in V_{t+1}^-}\brt{\mbf{h}^{(l)}_{t}(i)}\Big) \Big\rvert \\
&= \Big\lvert \Big(\frac{1}{\abs{V_{t+1}^+}}\sum_{i \in V_{t+1}^+}\mbf{h}_t^{(l)}(i) - \mbb{E}_{i | i \in V_{t+1}^+}\brt{\mbf{h}^{(l)}_{t}(i)} \Big) - \Big(\mbb{E}_{i | i \in V_{t+1}^-}\brt{\mbf{h}^{(l)}_{t}(i)} - \frac{1}{\abs{V_{t+1}^-}}\sum_{i \in V_{t+1}^-} \mbf{h}_t^{(l)}(i) \Big) \Big\rvert \\
&\leq \Big\lvert \Big(\frac{1}{\abs{V_{t+1}^+}}\sum_{i \in V_{t+1}^+}\mbf{h}_t^{(l)}(i) - \mbb{E}_{i | i \in V_{t+1}^+}\brt{\mbf{h}^{(l)}_{t}(i)} \Big) \Big\rvert - \Big\lvert \Big(\frac{1}{\abs{V_{t+1}^-}}\sum_{i \in V_{t+1}^-} \mbf{h}_t^{(l)}(i) - \mbb{E}_{i | i \in V_{t+1}^-}\brt{\mbf{h}^{(l)}_{t}(i)} \Big) \Big\rvert 
}

where the last inequality follows from the triangle inequality. Here, the first term amounts to the difference between the mean representation for positive nodes and the expected representation of postive nodes, while the second term is the difference between the mean representation for negative nodes and the expected representation of negative nodes. We proceed to bound each separately. For the first term, we obtain the following \textit{recursive formulation} for the hidden representation: 
\ml{\frac{1}{\abs{V_{t+1}^+}}\sum_{i \in V_{t+1}^+}\mbf{h}_t^{(l)}(i) - \mbb{E}_{i | i \in V_{t+1}^+}\brt{\mbf{h}^{(l)}_{t}(i)}
=\frac{1}{\abs{V_{t+1}^+}}\sum_{i \in V_{t+1}^+}\frac{1}{d_t(i)} \sum_{j \in \hat{\mc{N}}_t(i)} \mbf{h}_t^{(l-1)}(j) - \mbb{E}_{i | i \in V_{t+1}^+}\brt{\mbf{h}^{(l)}_{t}(i)} \\
=\frac{1}{\abs{V_{t+1}^+}}\sum_{i \in V_{t+1}^+}\frac{1}{d_t(i)} \sum_{j \in \hat{\mc{N}}_t(i)} \cdots \frac{1}{d_t(i)} \sum_{j \in \hat{\mc{N}}_t(i)} \phi(\mbf{x_t'(j)}) - \mbb{E}_{i | i \in V_{t+1}^+}\brt{\mbf{h}^{(l)}_{t}(i)},}

where $\phi(\mbf{x_t'(j)}) = \sigma^2_t(j) \mbf{x_t(j)} + \mu_t(j)$, $\mu_t(j)$ is the mean of node $j$, and $\sigma^2_t(j)$ is the variance of node $j$ at time $t$. First notice that each $\mbf{x_t'(j)}$ are \textit{standard} gaussian random variables. We have achieved this result by introducing function $\phi$, which transforms standard gaussian random variables to gaussian random variables with mean $\mu_t(j)$ and variance $\sigma^2_t(j)$. Second, we have that the empirical mean representation for the positive class is a function $f^+_t: \mbb{R}^n \to \mbb{R}$ of standard gaussian random variables. Thus, we can use Lemmas \ref{lem:lip_1} and \ref{lem:lip_2} to bound the distance between the empirical mean representation for positive nodes and the expected representation for positive nodes.
\al{f^{(l), +}_t(\mbf{x_t}) &= \frac{1}{\abs{V_{t+1}^+}}\sum_{i \in V_{t+1}^+}\mbf{h}_t^{(l)}(i) = \frac{1}{\abs{V_{t+1}^+}} (h_t^{(l)}(0) + \cdots + h_t^{(l)}(\abs{V_{t+1}^+})) \\
\lvert \lvert \nabla f^{(l), +}_t \rvert\rvert_2^2 &= \sum_{i \in V_{t+1}^+} \prn{\frac{\sigma^2}{\abs{V_{t+1}^+}} \sum_{j \in \hat{\mc{N}}_t(i)} \frac{l}{d_t(j)^l}}^2 = \frac{\sigma^4}{\abs{V_{t+1}^+}^2} \sum_{i \in V_{t+1}^+} \prn{\sum_{j \in \hat{\mc{N}}_t(i)} \frac{l}{d_t(j)^l}}^2 \label{eq:26}
}

where equation \eqref{eq:26} follows from the fact that $\mbf{x}_t(i)$ appears in the summation over all positive nodes precisely $d_t(i) \cdot l$ times, $l$ times for each of its neighbors $j \in \hat{\mc{N}}_t(i)$. Following a similar derivation for the second term where we aim to bound the difference between the mean negative node representation and the expected negative representation at layer $l$ yields:
\al{f^{(l), -}_t(\mbf{x_t}) &= \frac{1}{\abs{V_{t+1}^-}}\sum_{i \in V_{t+1}^+}\mbf{h}_t^{(l)}(i) = \frac{1}{\abs{V_{t+1}^-}} (h_t^{(l)}(0) + \cdots + h_t^{(l)}(\abs{V_{t+1}^-})) \\
\lvert \lvert \nabla f^{(l), -}_t \rvert\rvert_2^2 &= \sum_{i \in V_{t+1}^+} \prn{\frac{\sigma^2}{\abs{V_{t+1}^-}} \sum_{j \in \hat{\mc{N}}_t(i)} \frac{l}{d_t(j)^l}}^2 = \frac{\sigma^4}{\abs{V_{t+1}^-}^2} \sum_{i \in V_{t+1}^+} \prn{\sum_{j \in \hat{\mc{N}}_t(i)} \frac{l}{d_t(j)^l}}^2
}

By Lemma 1, we obtain the following bounds for (1) the distance between the empirical positive node representation and the expected positive representation at layer $l$ and (2) the distance between the empirical negative node representation and the expected negative representation at layer $l$: 
\al{\mbb{P}\Bigg(\Big\lvert \frac{1}{\abs{V_{t+1}^+}}\sum_{i \in V_{t+1}^+} \mbf{h}_t^{(l)}(i) - \mbb{E}_{i | i \in V_{t+1}^+}\brt{\mbf{h}^{(l)}_{t}(i)} \Big\rvert &\geq \epsilon \Bigg)\leq \mc{O}\Big(\text{exp}\Big(\frac{-\epsilon^2 \abs{V_{t+1}^+}^2}{\sigma^4\sum_{i \in V_{t+1}^+} \prn{ \sum_{j \in \hat{\mc{N}}_t(i)} \frac{l}{d_t(j)^l}}^2}\Big)\Big) \\
\mbb{P}\Bigg(\Big\lvert \frac{1}{\abs{V_{t+1}^-}}\sum_{i \in V_{t+1}^-} \mbf{h}_t^{(l)}(i) - \mbb{E}_{i | i \in V_{t+1}^-}\brt{\mbf{h}^{(l)}_{t}(i)} \Big\rvert &\geq \epsilon \Bigg)\leq \mc{O}\Big(\text{exp}\Big(\frac{-\epsilon^2 \abs{V_{t+1}^-}^2}{\sigma^4\sum_{i \in V_{t+1}^-} \prn{ \sum_{j \in \hat{\mc{N}}_t(i)} \frac{l}{d_t(j)^l}}^2}\Big)\Big)}

Finally, we apply a union bound to conclude the proof.
\end{proof}

\subsection{Expected AUROC for $l$ layer GNNs \label{pf:auc}}

\begin{theorem}
The expected AUROC of $f^{(l)}$ at timestep $t$ can be upper bounded as follows, 
\eq{\mbb{E}[A_t(f^{(l)})] \leq 1 - \Phi\prn{-\frac{2 \cdot \mu_t \cdot (h_t^+ + h_t^- - 1)^l}{v_{t+1}^+(l) + v_{t+1}^-(l)}}.}
where $v_{t+1}^+(l)$ and $v_{t+1}^-(l)$ are the lower bounds of the variances of the future positive and negative nodes after $l$ GCN layers, respectively, and are defined recursively in terms of the dynamic homophily levels as follows,
\al{v_{t+1}^+(l) &= {h_t^+}^2 \cdot v_{t+1}^+(l-1) + (1 - {h_t^+})^2 \cdot v_{t+1}^-(l-1) \\
v_{t+1}^-(l) &= {h_t^-}^2 \cdot v_{t+1}^-(l-1) + (1-{h_t^-})^2 \cdot v_{t+1}^+(l-1) \\
v_{t+1}^+(0) &= v_{t+1}^-(0) = \sigma^2.}
Moreover the probability that at time $t$ the distance between the empirical AUROC and the expected AUROC of a $l$ layer GCN is larger than $\epsilon$ is bounded as follows, 
\eq{\mbb{P}(\abs{A_t(f^{(l)}) - \mbb{E}[A_t(f^{(l)})]} \geq \epsilon) \leq e^{\frac{-2\cdot\abs{V_{t+1}^+}\cdot\abs{V_{t+1}^-}\cdot\epsilon^2}{\abs{V_{t+1}^+}+\abs{V_{t+1}^-}}}}
\end{theorem}

In order to prove Theorem \ref{thm:3}, our main goal will be to solve for the variances of the representations for the positive class and negative class respectively. Once we solve for the variances, we can use Lemma \ref{lem:2} and Theorem \ref{thm:1} to obtain the full expression.

\begin{proof}

We prove Theorem \ref{thm:3} with induction. The first step is to show that the node representations after 1 GCN layer can be expressed as a linear combination of independent Gaussian random variables. Beginning with 1-layer GNNs on the positive case, the variance terms can be written, 
\eq{\mbb{V}_{i | i \in V_{t+1}^+}\brt{\mbf{h}^{(1)}_{t}(i)} = \mbb{V}_{i | i \in V_{t+1}^+}\brt{\sum_{ j \in \hat{\mathcal{N}}(i) }\frac{\mbf{x}_t(j)}{d_t(i) + 1}} \label{eq:44}}

Equation \ref{eq:44} indeed tells us that after 1-layer of GCN, the node representation for the positive class is a linear combination of independent Gaussian random variables. Since there are only two different types of Gaussian random variables, one centered at $\mu_t$ and the other at $-\mu_t$, we can solve for the variance by determining the ratio of nodes from each of the distributions. To do so, we can apply the same partition trick as in Theorem \ref{thm:1}, treating $\mbf{h}^{(1)}_{t}(i)$ as a weighted sum of the Gaussian random variables where the weights are the dynamic homophily levels. Conditioning on $i \in V_{t+1}^+$, we have,
\al{\mbb{V}_{i | i \in V_{t+1}^+}\brt{\mbf{h}^{(1)}_{t}(i)} &= \mbb{V}_{i, j, k| i \in V_{t+1}^+, j\in V_t^+, k\in V_t^-}\brt{\mbf{x}_t(j) \cdot h_t^+ + \mbf{x}_t(k) \cdot (1 - h_t^+)}\\
&= {h_t^+}^2 \cdot \sigma^2 + (1 - h_t^+)^2 \sigma^2}

where the last line follows from $\mbf{x}_t(j)$ and $\mbf{x}_t(k)$ being Gaussian random variables with variance $\sigma^2$. Following a similar derivation for the negative case and conditioning on $i \in V_{t+1}^-$, we have, 
\al{\mbb{V}_{i | i \in V_{t+1}^-}\brt{\mbf{h}^{(1)}_{t}(i)} &= \mbb{V}_{i, j, k| i \in V_{t+1}^-, j\in V_t^+, k\in V_t^-}\brt{\mbf{x}_t(k) \cdot h_t^- + \mbf{x}_t(j) \cdot (1 - h_t^-)} \\
&= {h_t^-}^2 \cdot \sigma^2 + (1 - h_t^-)^2 \cdot \sigma^2}

The variances of the node representations for the one-layer case can then be written as, 
\eq{\mbb{V}_{i,j | i \in V_t^+, j \in V_t^-}[\mbf{h}^{(1)}_{t}(i) - \mbf{h}^{(1)}_{t}(j)] = {h_t^+}^2 \cdot \sigma^2 + (1 - h_t^+)^2 \sigma^2 + {h_t^-}^2 \cdot \sigma^2 + (1 - h_t^-)^2 \cdot \sigma^2,}

which satisfies the base case. We now consider $l+1$-layer GCNs. In the multilayer case conditioning on positive nodes, the key in determining the variance is to apply the partition trick as follows, 
\ml{\mbb{V}_{i | i \in V_{t+1}^+}\brt{\mbf{h}^{(l+1)}_{t}(i)} = \mbb{V}_{i, j, k| i \in V_{t+1}^+, j\in V_t^+, k\in V_t^-}\brt{\mbf{h}^{(l)}_t(j) \cdot h_t^+ + \mbf{h}^{(l)}_t(k) \cdot (1 - h_t^+)} \\
= {h_t^+}^2 \cdot \mbb{V}_{i | i \in V_{t+1}^+}\brt{\mbf{h}^{(l)}_{t}(i)} + (1 - h_t^+)^2 \cdot \mbb{V}_{j | j \in V_{t+1}^-}\brt{\mbf{h}^{(l)}_{t}(j)} + c\cdot \text{Cov}_{i, j | i\in V_{t+1}^+, j \in V_{t+1}^-}(\mbf{h}^{(l)}_{t}(i), \mbf{h}^{(l)}_{t}(j)),\label{eq:52} \\
\geq  {h_t^+}^2 \cdot \mbb{V}_{i | i \in V_{t+1}^+}\brt{\mbf{h}^{(l)}_{t}(i)} + (1 - h_t^+)^2 \cdot \mbb{V}_{j | j \in V_{t+1}^-}\brt{\mbf{h}^{(l)}_{t}(j)} }

where $c=2{h_t^+}^2(1 - h_t^+)^2$ and in the first line we have partitioned $\mbf{h}^{(l+1)}_{t}(i)$ into a weighted sum of $\mbf{h}^{(l)}_t(j)$ and $\mbf{h}^{(l)}_t(k)$ based on the positive dynamic homophily. In the second line, we have an additional covariance term since $\mbf{h}^{(l)}_{t}(i)$ and $\mbf{h}^{(l)}_{t}(j)$ are not necessarily independent and the two hidden representations may share nodes. We obtain the lower bound by noticing that the covariance term must be positive due to the linearity of covariance. In essence, $\mbf{h}^{(l)}_{t}(i)$ and $\mbf{h}^{(l)}_{t}(j)$ are weighted sums of independent gaussian random variables. Expanding the covariance through linearity, the only nonzero terms are the variances of the shared nodes between the two representations and since the variance must be positive, the covariance must also be positive. Now, conditioning on negative nodes, we have, 
\ml{\mbb{V}_{i | i \in V_{t+1}^-}\brt{\mbf{h}^{(l+1)}_{t}(i)} = \mbb{V}_{i, j, k| i \in V_{t+1}^+, j\in V_t^+, k\in V_t^-}\brt{\mbf{h}^{(l)}_t(k) \cdot h_t^- + \mbf{h}^{(l)}_t(j) \cdot (1 - h_t^-)} \\
= {h_t^-}^2 \cdot \mbb{V}_{i | i \in V_{t+1}^-}\brt{\mbf{h}^{(l)}_{t}(i)} + (1 - h_t^-)^2 \cdot \mbb{V}_{j | j \in V_{t+1}^+}\brt{\mbf{h}^{(l)}_{t}(j)} + c\cdot\text{Cov}_{i, j | i\in V_{t+1}^+, j \in V_{t+1}^-}(\mbf{h}^{(l)}_{t}(i), \mbf{h}^{(l)}_{t}(j)) \label{eq:53} \\
\geq {h_t^-}^2 \cdot \mbb{V}_{i | i \in V_{t+1}^-}\brt{\mbf{h}^{(l)}_{t}(i)} + (1 - h_t^-)^2 \cdot \mbb{V}_{j | j \in V_{t+1}^+}\brt{\mbf{h}^{(l)}_{t}(j)}}

where $c=2{h_t^-}^2(1 - h_t^-)^2$. Adding the variances in Equations \ref{eq:52} and \ref{eq:53} and applying the induction hypothesis satisfies our recursive definitions of the variance lower bounds for $l+1$-layer GCNs. In order to obtain the probability bound, we apply Theorem 5 in \citet{agarwal2005generalization}, bounding the distance between the expected AUROC and the empirical AUROC.

\end{proof}

\subsection{Expected Distance and Concentration in Multiclass Classification \label{pf:auc}} 

\begin{theorem}
Given graph $\bs{G}_{0:T}$ at timestep $t$, the Euclidean distance of the expected difference between a randomly sampled node of class $c_m$ and a randomly sampled node of class $c_n$ after 1 layer of a GCN can be expressed as: 
\eq{\Big\lvert \Big\lvert \mbb{E}_{i | y_{t+1}(i)=c_m}\brt{\mbf{h}^{(1)}_{t}(i)} - \mbb{E}_{i | y_{t+1}(i)=c_n}\brt{\mbf{h}^{(1)}_{t}(i)} \Big\rvert \Big\rvert_2
= \mu_t \cdot \prn{\sum_{k=1}^K (\mbf{T}_t[c_m, c_k] - \mbf{T}_t[c_n, c_k])^2}^{\frac{1}{2}}.}
Furthermore, for any $\epsilon > 0$, the probability that the distance between the empirical mean representation and the expected representation is larger than $\epsilon$ is bounded as follows:
\ml{\mbb{P}\prn{\Big\lvert\Big\lvert \Big(\frac{1}{\abs{V_{t+1}^{c_m}}}\sum_{i \in V_{t+1}^{c_m}} \mbf{h}_t^{(1)}(i) - \frac{1}{\abs{V_{t+1}^{c_n}}}\sum_{i \in V_{t+1}^{c_n}} \mbf{h}_t^{(1)}(i)\Big) - \Big(\mbb{E}_{i |i \in V_{t+1}^{c_m}}\brt{\mbf{h}^{(1)}_{t}(i)} -\mbb{E}_{i | i \in V_{t+1}^{c_n}}\brt{\mbf{h}^{(1)}_{t}(i)}\Big) \Big\rvert\Big\rvert_2 \geq \epsilon} \\ 
\leq \mc{O}\prn{\text{exp}\prn{\frac{-\epsilon^2 \abs{V_{t+1}^{c_m}}^2 \lvert C \rvert}{\sum_{i \in V_{t+1}^{c_m}} \sum_{j \in \hat{\mc{N}}_t(i)} \frac{1}{d_t(j)}}^2}+ \text{exp}\prn{\frac{-\epsilon^2 \abs{V_{t+1}^{c_n}}^2 \lvert C \rvert}{\sum_{i \in V_{t+1}^{c_n} } \prn{\sum_{j \in \hat{\mc{N}}_t(i)} \frac{1}{d_t(j)}}^2}}}
}
\end{theorem}

\begin{proof}
In the multiclass classification case, rather than partitioning the neighbors into those whose label at timestep $t$ is the same as $i$'s label at timestep $t+1$, we instead partition the neighbors into their respective classes at timestep $t$. Now, the partition is determined by entries of the dynamic compatibility matrix. First assume $y_{t+1}(i)=c_m$: 
\al{\mbb{E}_{i | i \in V_{t+1}^{c_m}}\brt{\mbf{h}^{(1)}_{t}(i)}&=\sum_{k=1}^K \mbb{E}_{j | y_t(j)=c_k} \brt{\mbf{x}_t(j)} 
\cdot \mbb{P}(y_t(j)=c_k \mid j \in \hat{\mc{N}}_t(i), y_{t+1}(i) = c_m) \\
&=\sum_{k=1}^K \mbb{E}_{j | y_t(j)=c_k} \brt{\mbf{x}_t(j)}\cdot \mbf{T}_t[c_m, c_k] \label{eq:49}.} 
 
We follow a similar derivation assuming $y_{t+1}(i)=c_n$, and we obtain the following:
\al{\mbb{E}_{i | i \in V_{t+1}^{c_n}}\brt{\mbf{h}^{(1)}_{t}(i)} &=\sum_{k=1}^K \mbb{E}_{j | y_t(j)=c_k} \brt{\mbf{x}_t(j)} \cdot \mbb{P}(y_t(j)=c_k \mid j \in \hat{\mc{N}}_t(i), y_{t+1}(i) = c_n) \\
&=\sum_{k=1}^K \mbb{E}_{j | y_t(j)=c_k} \brt{\mbf{x}_t(j)} \cdot \mbf{T}_t[c_n, c_k] \label{eq:51}.}

Measuring the $l^2$ norm of the difference between \eqref{eq:51} and \eqref{eq:49}, and applying the assumption that for all $i$, $y_t(i) \in \mc{C}$ and $\mbf{x}_t(i) \sim \mc{N}(\bs{\mu_t}_{\mbf{Y}_t(i)}, \sigma^2)$, yields:
\al{\Big\lvert \Big\lvert \mbb{E}_{i | i \in V_{t+1}^{c_m}}\brt{\mbf{h}^{(1)}_{t}(i)} -\mbb{E}_{i | i \in V_{t+1}^{c_n}}\brt{\mbf{h}^{(1)}_{t}(i)} \Big\rvert \Big\rvert_2 &= \mu_t \cdot \Big\lvert \Big\lvert \sum_{k=1}^K (\mbf{T}_t[c_m, c_k] - \mbf{T}_t[c_n, c_k]) \Big\rvert \Big\rvert_2 \label{eq:27} \\
&= \mu_t \cdot \prn{\sum_{k=1}^K (\mbf{T}_t[c_m, c_k] - \mbf{T}_t[c_n, c_k])^2}^{\frac{1}{2}}.}

We aim to bound the difference between the expected difference in node representation and their empirical difference. To prove the bound, we use a similar derivation as in Theorems \ref{thm:2}. We first decompose the difference terms and treat each separately as follows:
\ml{\Big\lvert\Big\lvert \Big(\frac{1}{\abs{V_{t+1}^{c_m}}}\sum_{i \in V_{t+1}^{c_m}}\mbf{h}_t^{(1)}(i) - \frac{1}{\abs{V_{t+1}^{c_n}}}\sum_{i \in V_{t+1}^{c_n}} \mbf{h}_t^{(1)}\Big) - \Big(\mbb{E}_{i | i \in V_{t+1}^{c_m}}\brt{\mbf{h}^{(1)}_{t}(i)} -\mbb{E}_{i | i \in V_{t+1}^{c_n}}\brt{\mbf{h}^{(1)}_{t}(i)}\Big) \Big\rvert \Big\rvert_2 \\
\leq \Big\lvert\Big\lvert \Big(\frac{1}{\abs{V_{t+1}^{c_m}}}\sum_{i \in V_{t+1}^{c_m}}\mbf{h}_t^{(1)}(i) - \mbb{E}_{i | i \in V_{t+1}^{c_m}}\brt{\mbf{h}^{(1)}_{t}(i)} \Big) \Big\rvert\Big\rvert_2 - \Big\lvert\Big\lvert \Big(\frac{1}{\abs{V_{t+1}^{c_n}}}\sum_{i \in V_{t+1}^{c_n}} \mbf{h}_t^{(1)}(i) - \mbb{E}_{i | i \in V_{t+1}^{c_n}}\brt{\mbf{h}^{(1)}_{t}(i)} \Big) \Big\rvert\Big\rvert_2}

where the last inequality follows from the triangle inequality. Again, the first term amounts to the difference between the mean representation for nodes in $V_{t+1}^{c_m}$ and the expected representation of nodes in nodes in $V_{t+1}^{c_m}$, while the second term is the difference between the mean representation for nodes in nodes in $V_{t+1}^{c_n}$ and the expected representation of nodes in nodes in $V_{t+1}^{c_m}$. We proceed to bound each separately.

The key in proving the upper bound in Theorem \ref{thm:6} is to handle each dimension of the node representations separately. Using our assumption that node representations are drawn from a multivariate gaussian with a diagonal covariance, each dimension of the representation is independent and equivalent to a univariate gaussian. For dimension arbitrary dimension $k$ of the representation, we obtain the following \textit{recursive formulation} for the hidden representation: 
\ml{\frac{1}{\abs{V_{t+1}^{c_m}}}\sum_{i \in V_{t+1}^{c_m}}\mbf{h}_t^{(l)}(i)[k] - \mbb{E}_{i | i \in V_{t+1}^{c_m}}\brt{\mbf{h}^{(l)}_{t}(i)[k]}\\ 
=\frac{1}{\abs{V_{t+1}^{c_m}}}\sum_{i \in V_{t+1}^{c_m}}\frac{1}{d_t(i)} \sum_{j \in \hat{\mc{N}}_t(i)} \mbf{x}_t^{(1)}(j)[k] - \mbb{E}_{i | i \in V_{t+1}^{c_m}}\brt{\mbf{h}^{(1)}_{t}(i)[k]} \\
=\frac{1}{\abs{V_{t+1}^{c_m}}}\sum_{i \in V_{t+1}^{c_m}}\frac{1}{d_t(i)} \sum_{j \in \hat{\mc{N}}_t(i)} \phi(\mbf{x_t'(j)[k]}) - \mbb{E}_{i | i \in V_{t+1}^{c_m}}\brt{\mbf{h}^{(l)}_{t}(i)[k]}}

where $\phi(\mbf{x_t'(j)})$, $\mu_t(j)$, and $\sigma^2_t(j)$ are as defined in Theorem \ref{thm:1}. Again, each $\mbf{x_t'(j)[k]}$ are \textit{standard} gaussian random variables, and the empirical mean representation for the positive class at layer $l$ is a function $f^{{(l)}, +}_t: \mbb{R}^n \to \mbb{R}$ of standard gaussian random variables. Thus, we can use Lemmas \ref{lem:lip_1} and \ref{lem:lip_2} to bound the distance between the empirical mean representation for positive nodes and the expected representation for positive nodes.
\al{f^{(1), c_m}_t(\mbf{x_t}) &= \frac{1}{\abs{V_{t+1}^{c_m}}}\sum_{i \in V_{t+1}^{c_m}}\mbf{h}_t^{(1)}(i)[k] = \frac{1}{\abs{V_{t+1}^{c_m}}} (h_t^{(1)}(0)[k] + \cdots + h_t^{(1)}(\abs{V_{t+1}^{c_m}})[k]) \\
\lvert \lvert \nabla f^{(l), c_m}_t \rvert\rvert_2^2 &= \sum_{i \in V_{t+1}^{c_m}} \prn{\frac{\sigma^2}{\abs{V_{t+1}^{c_m}}} \sum_{j \in \hat{\mc{N}}_t(i)} \frac{1}{d_t(j)}}^2 = \frac{\sigma^4}{\abs{V_{t+1}^{c_m}}^2} \sum_{i \in V_{t+1}^{c_m}} \prn{\sum_{j \in \hat{\mc{N}}_t(i)} \frac{1}{d_t(j)}}^2 \label{eq:41}
}

Following a similar derivation for the second term where we aim to bound the difference between the mean negative node representation and the expected negative representation at layer $l$ yields:
\al{f^{(1), c_n}_t(\mbf{x_t}) &= \frac{1}{\abs{V_{t+1}^{c_n}}}\sum_{i \in V_{t+1}^{c_n}}\mbf{h}_t^{(1)}(i)[k] = \frac{1}{\abs{V_{t+1}^{c_n}}} (h_t^{(1)}(0)[k] + \cdots + h_t^{(l)}(\abs{V_{t+1}^{c_n}}[k])) \\
\lvert \lvert \nabla f^{(1), c_n}_t \rvert\rvert_2^2 &= \sum_{i \in V_{t+1}^{c_n}} \prn{\frac{\sigma^2}{\abs{V_{t+1}^{c_n}}} \sum_{j \in \hat{\mc{N}}_t(i)} \frac{1}{d_t(j)}}^2 = \frac{\sigma^4}{\abs{V_{t+1}^{c_n}}^2} \sum_{i \in V_{t+1}^{c_m}} \prn{\sum_{j \in \hat{\mc{N}}_t(i)} \frac{l}{d_t(j)}}^2
}

By Lemma 1, we obtain the following bounds for (1) the distance between the empirical positive node representation and the expected positive representation at layer $l$ and (2) the distance between the empirical negative node representation and the expected negative representation at layer $l$: 
\al{\mbb{P}\Bigg(\Big\lvert \frac{1}{\abs{V_{t+1}^{c_m}}}\sum_{i \in V_{t+1}^{c_m}} \mbf{h}_t^{(1)}(i) - \mbb{E}_{i | i \in V_{t+1}^{c_m}}\brt{\mbf{h}^{(1)}_{t}(i)} \Big\rvert &\geq \epsilon \Bigg)\leq \mc{O}\Big(\text{exp}\Big(\frac{-\epsilon^2 \abs{V_{t+1}^{c_m}}^2}{\sigma^4 \sum_{i \in V_{t+1}^{c_m}} \prn{\sum_{j \in \hat{\mc{N}}_t(i)} \frac{l}{d_t(j)^l}}^2}\Big)\Big) \\
\mbb{P}\Bigg(\Big\lvert \frac{1}{\abs{V_{t+1}^{c_n}}}\sum_{i \in V_{t+1}^{c_n}} \mbf{h}_t^{(1)}(i) - \mbb{E}_{i | i \in V_{t+1}^{c_n}}\brt{\mbf{h}^{(1)}_{t}(i)} \Big\rvert &\geq \epsilon \Bigg)\leq \mc{O}\Big(\text{exp}\Big(\frac{-\epsilon^2 \abs{V_{t+1}^{c_n}}^2}{\sigma^4 \sum_{i \in V_{t+1}^{c_n}} \prn{\sum_{j \in \hat{\mc{N}}_t(i)} \frac{l}{d_t(j)^l}}^2}\Big)\Big)}

Notice, we have obtained our bounds assuming arbitrary dimension $k$. It suffices to apply a second union bound over all $k \in \lvert C \rvert$ in order to obtain the final bound in Theorem \ref{thm:6}.
\end{proof}

\section{ADDITIONAL THEORY, EMPIRICAL RESULTS, AND DISCUSSIONS}\label{app:disc}

\subsection{Dynamic Homophily and Multiclass Classification \label{app:mult}}

In the multiclass classification case, dynamic homophily alone does not represent the distance between nodes of different classes. Thus, we introduce the dynamic compatibility matrix as an extension of the class compatibility matrix, allowing us to estimate the node representations of GCNs and measure the expected distance in node representations.

\textbf{Why is static homophily insufficient?} We first consider why static homophily is insufficient in multiclass classification on a static graph. In multiclass classification, each class can be associated with $\abs{\mc{C}}$ probabilities, where the $k$-th probability for class $m$ denotes the probability a node of class $m$ forms an edge with a node of class $k$. These probabilities are essential for estimating node representations because they capture the full neighbor label distribution across all node classes. Intuitively, each class must have a unique neighbor distribution in order for GNNs to discriminate between the classes. Notably, here static homophily can be low, yet the distributions can be unique, leading to high GNN performance \citep{lim2021large,ma2022is,zhu2023heterophily}. 

In the the static setting, the class compatibility matrix captures the entire neighbor label distribution for each class. Thus, we aim to develop a similar quantity in the dynamic setting that captures the same intuition as dynamic class homophily in binary classification. To this end, we propose the dynamic compatibility matrix formally defined as follows: 

\begin{definition}[Dynamic Compatibility Matrix]
Given dynamic graph $\bs{G}_{0:T}$, the dynamic compatibility matrix $\mbf{T}_t$ at timestep $t$ is a $\abs{\mc{C}} \times \abs{\mc{C}}$ matrix where entry $\mbf{T}_t[c_m , c_n]$ is defined: 
\eq{\mbf{T}_t[c_m, c_n] = \mbb{P}(y_t(j)=c_n \mid j \in \hat{\mc{N}}_t(i), y_{t+1}(i) = c_m).}
\end{definition}

We now proceed to measure the node representations for specific classes under multiclass node classification in the dynamic setting. Here, we show that the expected Euclidean distance between node representations of different classes is a function of the Euclidean distance between the neighbor label distributions of the two classes. In our analysis we make the following assumptions: 

\textbf{Assumptions (Multiclass Classification):} Let $\bs{G}_{0:T}$ be a dynamic graph, and $\mc{C} = \{c_1, \ldots, c_K \}$ be the set of node classes, where $K > 2$. For all $t \in [0, T]$ and for all $i \in V_t$, $y_t(i) \in \mc{C}$ and $\mbf{x}_t(i) \sim N(\bs{\mu_t}_{y_t(i)}, \sigma^2)$ where $N$ is a mutivariate normal distribution with mean $\bs{\mu_t}_{y_t(i)} = \mu_t \cdot \text{one-hot}(y_t(i)) \in \mbb{R}^K$ and covariance $\sigma^2 = \text{Diag}(\bs{\sigma^2}) \in \mbb{R}^{K \times K}$. 

\begin{theorem}
The Euclidean distance of the expected difference between a randomly sampled node of class $c_m$ and a randomly sampled node of class $c_n$ after 1 layer of a GCN can be expressed as: 
\eq{\Big\lvert \Big\lvert \mbb{E}_{i | y_{t+1}(i)=c_m}\brt{\mbf{h}^{(1)}_{t}(i)} - \mbb{E}_{i | y_{t+1}(i)=c_n}\brt{\mbf{h}^{(1)}_{t}(i)} \Big\rvert \Big\rvert_2
= \mu_t \cdot \prn{\sum_{k=1}^K (\mbf{T}_t[c_m, c_k] - \mbf{T}_t[c_n, c_k])^2}^{\frac{1}{2}}.}
Furthermore, for any $\epsilon > 0$, the probability that the distance between the empirical mean representation and the expected representation is larger than $\epsilon$ is bounded as follows:
\ml{\mbb{P}\prn{\Big\lvert\Big\lvert \Big(\frac{1}{\abs{V_{t+1}^{c_m}}}\sum_{i \in V_{t+1}^{c_m}} \mbf{h}_t^{(1)}(i) - \frac{1}{\abs{V_{t+1}^{c_n}}}\sum_{i \in V_{t+1}^{c_n}} \mbf{h}_t^{(1)}(i)\Big) - \Big(\mbb{E}_{i |i \in V_{t+1}^{c_m}}\brt{\mbf{h}^{(1)}_{t}(i)} -\mbb{E}_{i | i \in V_{t+1}^{c_n}}\brt{\mbf{h}^{(1)}_{t}(i)}\Big) \Big\rvert\Big\rvert_2 \geq \epsilon} \\ 
\leq \mc{O}\prn{\text{exp}\prn{\frac{-\epsilon^2 \abs{V_{t+1}^{c_m}}^2 \lvert C \rvert}{\sum_{i \in V_{t+1}^{c_m}} \sum_{j \in \hat{\mc{N}}_t(i)} \frac{1}{d_t(j)}}^2}+ \text{exp}\prn{\frac{-\epsilon^2 \abs{V_{t+1}^{c_n}}^2 \lvert C \rvert}{\sum_{i \in V_{t+1}^{c_n} } \prn{\sum_{j \in \hat{\mc{N}}_t(i)} \frac{1}{d_t(j)}}^2}}}
}
\label{thm:6}
\end{theorem}

Theorem \ref{thm:6} tells us that the expected Euclidean distance between node representations of different classes depend on the Euclidean distance between the neighbor label distributions of the two classes. Moreover, the expected distance is close to the empirical distance with high probability. As the neighbor label distributions for two classes become more similar, the Euclidean distance between their expected node representation decreases, and we expect 1-layer GCNs to perform worse in discriminating between the two classes. 

\subsection{Extension of Dynamic homophily to Continuous Dynamic Graphs \label{app:cont}}

In the main paper, we propose dynamic homophily in the context of discrete dynamic graphs. While there are many examples of discrete dynamic graphs such as social networks and protein-protein interaction networks, there are also continuous representations of graphs where the graph is defined as a sequence of edges rather than a sequence of static graphs. A straightforward way to extend dynamic homophily to continuous graphs is to first specify a window size $k$, then define the dynamic homophily at time $t$ as the dynamic homophily measured within the window size [$t$, $t + k$]. We obtain the sequence of dynamic homophily levels by measuring it along the windows [$tk$, $(t+1)k$] for all $t$. This strategy aligns with most tasks defined on continuous graphs, where performance is measured along time slices of the continuous graph and the window size is determined by the application \citep{kumar2019predicting, huang2023temporal}. 

In practice, the evaluation of dynamic homophily in continuous settings is challenging since most publicly available continuous graphs are heterogeneous, where nodes are of different types \citep{kumar2019predicting, huang2023temporal}. For example, \citet{kumar2019predicting} considers a Reddit network where nodes are users and subreddits, while edges connect users to subreddits. Another example can be found in \citet{huang2023temporal} where the dataset is a cryptocurrency network where nodes are users and tokens, while edges connect users to tokens. In these cases, it is not straightforward to measure homophily since it is difficult to compare user nodes to token or subreddit nodes, and there may be no interpretation for a user node to be similar to a token or subreddit node. Thus, due to the heterogeneity of these datasets, we leave the further exploration of heterogenous continuous dynamic graphs and dynamic homophily for future work.

\subsection{Further Discussion of Theorem \ref{thm:3} \label{sub:auc}}

Here, we provide additional discussion of Theorem \ref{thm:3}. In Figure \ref{fig:6}, we visualize the expected AUROC at different GCN layers on the same plot for different choices of $\mu_t$ and $\sigma^2$. We find that indeed optimal AUROC for different dynamic homophily levels is obtained for different GCN layers. In particular, when $\mu_t=2$ and $\sigma^2=.5$, deeper odd-layered GCNs lead to the best performance when both dynamic homophily levels are high, while deeper even-layered GCNs lead to the best performance when both dynamic homophily levels are low. When $\mu_t=1$ and $\sigma^2=1$, deeper GCNs perform poorly since smoothing is not as beneficial. Here, the optimal GCNs are 1-layer and 2-layer GCNs when dynamic homophily levels are both high and low, respectively. In light of these results, a natural solution in obtaining the best performance across a dynamic graph with changing dynamic and dynamic homophily levels is to combine the representations across layers such that a GNN leverages the most useful ones across the spectrum of dynamic and dynamic homophily levels.

\begin{figure*}[t!]
    \centering
    \includegraphics[width=.9\textwidth]{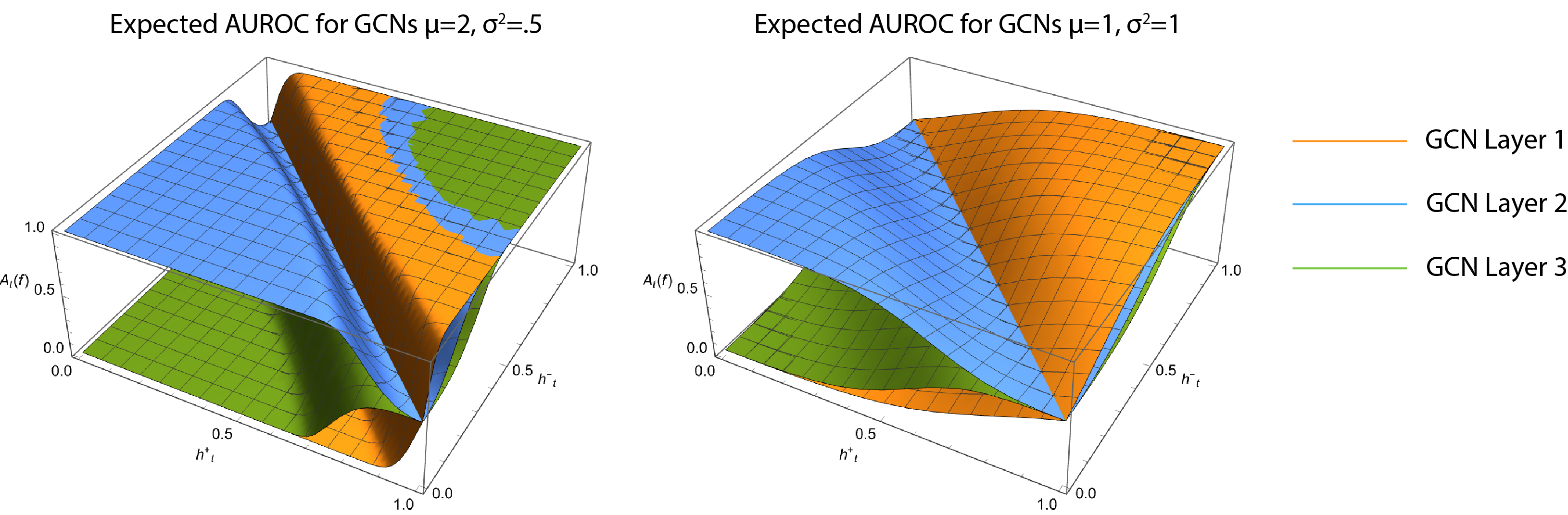}
    \caption{\footnotesize Expected AUROC across GCN layers as a function of dynamic homophily levels for different $\mu_t$ and $\sigma$.}
    \label{fig:6}
\end{figure*}

\subsection{Dynamic Homophily and Dynamic GNNs \label{app:dyn}}

Following the same training and evaluation setup described in the main paper, we also test dynamic homophily's correlation with performance for various dynamic GNNs. Generally, dynamic GNNs find a way to combine an RNN component with a message passing module in order to leverage the temporal signal present in the dynamic graph. However, since in our synthetic experimental setup the SI model makes the Markov assumption, a dynamic GNN that leverages the full history should simply learn to ignore all timesteps prior to $t$, making them equivalent to static GNNs on these datasets. On the Higgs networks, intuitively, the information from timestep $t$ should be enough information to predict the spread of the signal at time $t+1$. Lastly, existing results in \citet{fu2022dppin} on the protein-protein interaction networks suggest that dynamic GNNs leveraging the temporal signal perform just as well as static GNNs that do not leverage the temporal data. Given these arguments, we do not expect dynamic GNNs to perform better than static GNNs on our chosen datasets. Nevertheless, for completeness purposes, we test the following representative dynamic GNNs from the literature on a subset of the datasets in the main paper: \textbf{DGNN} \citep{manessi2020dynamic, narayan2018learning, chen2022gc}, \textbf{GCRN} \citep{seo2018structured}, and \textbf{EvolveGCN} \citep{pareja2020evolvegcn}. We provide descriptions of all dynamic GNNs in the context of their dynamic components in Appendix \ref{app:exps}.

Following our evaluation procedure in the main paper, we compute the average correlation between dynamic homophily and dynamic GNN performance. Across all combinations of synthetic graphs and GNN approaches, the average correlation between dynamic homophily and GNN performance exceeds the correlation between static homophily and GNN performance (Table \ref{tab:4}). To fully capture changes to dynamic GNN performance, we evaluate at each time step for each of the datasets (Figure \ref{fig:7}). More specifically we measure AUROC, dynamic homophily, and static homophily at each timestep for the Regular, Powerlaw, and Higgs 1 graph. While dynamic homophily exhibits the same trends as GNN performance, static homophily does not since it stays high across all three graphs for all timesteps. Comparing the performance of dynamic GNNs to static GNNs, we find generally that dynamic GNNs perform worse than static ones, confirming our hypothesis that GNNs leveraging the full temporal signal do not provide benefits on our chosen datasets (Table \ref{tab:5}).

\begin{figure*}[h!]
     \centering
     \begin{subfigure}[b]{.325\textwidth}
         \centering
         \includegraphics[width=\textwidth]{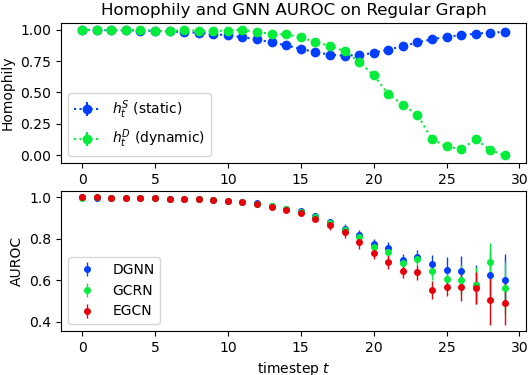}
         \caption{Regular dynamic graph}
     \end{subfigure}
     \hfill
     \begin{subfigure}[b]{.325\textwidth}
         \centering
         \includegraphics[width=\textwidth]{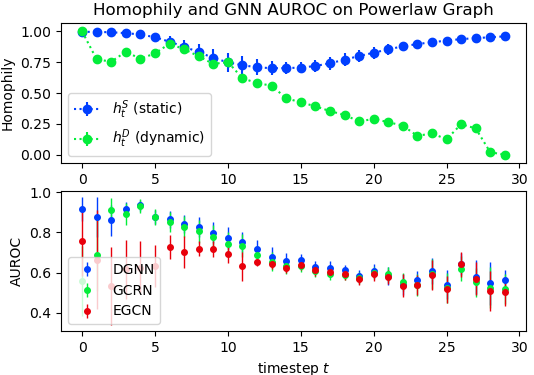}
         \caption{Powerlaw dynamic graph}
     \end{subfigure}
     \hfill
     \begin{subfigure}[b]{.325\textwidth}
         \centering
         \includegraphics[width=\textwidth]{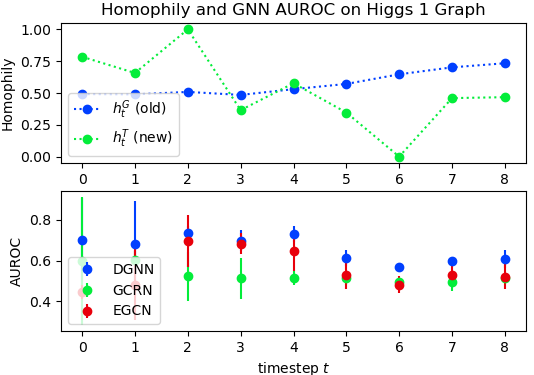}
         \caption{Higgs 1 dynamic graph}
     \end{subfigure}
    \caption{\footnotesize Mean and standard deviations (error bars) of dynamic homophily, static homophily, and AUROC across all graphs in the test set for the Regular, Powerlaw, and Higgs 1 dynamic graph. For all three dynamic graphs, static homophily stays relatively high across time, while dynamic homophily and GNN performances show different trends.}
    \label{fig:7}
\end{figure*}

\begin{table*}[h!]
\begingroup
\setlength{\tabcolsep}{6pt} 
\footnotesize
\caption{\footnotesize The mean ± standard deviation of Spearman's rank correlation coefficient between GNN AUROC and homophily measures across all graphs in the test set. For all combinations of GNN and dynamic graph, dynamic homophily tends to have a higher correlation with dynamic GNN performance compared to static homophily.}
\label{tab:4}
\begin{center}
\begin{tabular}{ l l @{\quad} c c l@{\quad} c c } 
\toprule
& & \multicolumn{2}{c }{\textbf{Epidemiological Dynamic Graphs}} && \multicolumn{2}{c }{\textbf{Higgs Dynamic Graphs}} \\
\cmidrule{3-4} \cmidrule{6-7}
 \textbf{GNN} & \textbf{Homophily} & \textbf{Regular}  & \textbf{Powerlaw} && \textbf{Higgs 1} & \textbf{Higgs 2} \\ 
  \midrule
 \multirow{2}{*}{\textbf{DGNN}} & (\textit{static}) $h_t^S$ & 0.52 ± 0.09 & 0.09 ± 0.15 && -0.55 ± 0.32 & OOM \\
 & (\textit{dynamic}) $h_t^D$ & \textbf{0.90 ± 0.04} & \textbf{0.83 ± 0.08} && \textbf{0.60 ± 0.35} & OOM \\
 \arrayrulecolor{gray}\hline
  \multirow{2}{*}{\textbf{GCRN}} & (\textit{static}) $h_t^S$ & 0.51 ± 0.09 & 0.13 ± 0.13 && -0.24 ± 0.26 & OOM \\
  & (\textit{dynamic}) $h_t^D$ & \textbf{0.91 ± 0.04} & \textbf{0.85 ± 0.07} && \textbf{0.22 ± 0.09} & OOM \\
 \hline
 \multirow{2}{*}{\textbf{EvolveGCN}} & (\textit{static}) $h_t^S$ & 0.50 ± 0.09 & -0.18 ± 0.23 && -0.05 ± 0.27 & OOM \\
 & (\textit{dynamic}) $h_t^D$ & \textbf{0.91 ± 0.03} & \textbf{0.77 ± 0.13} && \textbf{-0.03 ± 0.44} & OOM \\
 \arrayrulecolor{black}\bottomrule
\end{tabular}
\end{center}
\endgroup
\end{table*}

\begin{table*}[h!]
\begingroup
\setlength{\tabcolsep}{6pt} 
\footnotesize
\caption{\footnotesize Mean ± standard deviation of GNN AUROC across time for each dynamic graph in the test set. Dynamic graphs tend to perform worse than static ones, indicating leveraging the full timeseries on our chosen datasets do not improve performance.}
\label{tab:5}
\begin{center}
\begin{tabular}{ l @{\quad} c c l@{\quad} c c } 
\toprule
& \multicolumn{2}{c }{\textbf{Epidemiological Dynamic Graphs}} && \multicolumn{2}{c }{\textbf{Higgs Dynamic Graphs}} \\
\cmidrule{2-3} \cmidrule{5-6}
 \textbf{GNN} & \textbf{Regular}  & \textbf{Powerlaw} && \textbf{Higgs 1} & \textbf{Higgs 2} \\ 
  \midrule
 \textbf{GCN} & 0.96 $\pm$ 0.02 & 0.81 $\pm$ 0.01 && 0.56 $\pm$ 0.01 & 0.54 $\pm$ 0.01 \\
 \textbf{GAT} & 0.95 $\pm$ 0.02 & 0.78 $\pm$ 0.01 && 0.50 $\pm$ 0.01 & 0.50 $\pm$ 0.01 \\
 \textbf{SAGE} & \textbf{0.96 $\pm$ 0.01} & \textbf{0.82 $\pm$ 0.01} && 0.61 $\pm$ 0.00 & \textbf{0.55 $\pm$ 0.00} \\
 \midrule
 \textbf{DGNN} & 0.86 ± 0.01 & 0.69 ± 0.01 && \textbf{0.66 ± 0.07} & OOM \\
 \textbf{GCRN} & 0.86 ± 0.01 & 0.67 ± 0.02 && 0.53 ± 0.02 & OOM \\
 \textbf{EvolveGCN} & 0.83 ± 0.01 & 0.62 ± 0.01 && 0.56 ± 0.05 & OOM \\
 \arrayrulecolor{black}\bottomrule
\end{tabular}
\end{center}
\endgroup
\end{table*}

\section{EXPERIMENTAL DETAILS \label{app:exps}}

\subsection{Pseudo-synthetic Dataset Details}
We first describe how we utilize the SI model in generating the dynamic graph for the pseudo-synthetic dataset. Formally, at timestep $t$ susceptible node $i$ can become infected by an infected neighbor $j$ with probability $\alpha_j \cdot \beta_i$, where $\alpha_j \in [0, 1]$ is the infectivity of $j$ and $\beta_i \in [0, 1]$ is the susceptibility of $i$. Then, at timestep $t$ for susceptible node $i$ its label at timestep $t + 1$ is sampled from a Bernoulli parameterized by $1 - \prod_{j \in \mc{N}_t(i)} (1 - \alpha_j \cdot \beta_i \cdot \mbb{I}(y_t(j) == +1))$.

Given the graph structure, labels are generated by first assigning to each node hidden parameters $\alpha_i$ and $\beta_i$, sampled from unique beta distributions specified prior to assignment that do not change over time. Parameters for the beta distributions are selected such that all nodes in the dynamic graph are infected at the final timestep. We then randomly sample a small set of nodes at timestep $t=0$ as initially infected, determining $\mbf{y}_0$. Then, we use the SI model to generate the label sequence $\{\mbf{y}_1, \ldots, \mbf{y}_T\}$. At time $t$, node $i$'s feature vector is given by, $\mbf{x}_t(i) = \brt{y_t(i) || \alpha(i) || \beta(i)}$. 

We now present the pseudocode for generation of each dynamic graph in Algorithm 1. Below, we summarize the label and feature generation process for an arbitrary node $i$: 
\begin{align}
\alpha_i &\sim \mc{B}eta(\theta_{1,\text{inf}}, \theta_{2,\text{inf}}) \\
\beta_i &\sim \mc{B}eta(\theta_{1,\text{sus}}, \theta_{2,\text{sus}}) \\
\mbf{x}_t(i) &= \brt{y_t(i) || \alpha(i) || \beta(i)}
\end{align}
\eq{
y_{t+1}(i) \sim 
\begin{cases}
B(p_\text{init}), \hspace{.5em} &\text{if } t = 0\\ 
B\Big(1 - \hspace{-1.25em} \prod\limits_{j \in \hat{\mc{N}}_{t-1}(i)} (1 - \alpha_j \cdot \beta_i \cdot \mbb{I}(y_{t}(j) = +1))\Big), &\text{if } t>0
\end{cases}}

where, $\mc{B}eta(\theta_{1,\text{inf}}, \theta_{2,\text{inf}})$ and $\mc{B}eta(\theta_{1,\text{sus}}, \theta_{2,\text{sus}})$ represents the beta distributions for the infectivity and susceptibility parameters, and $B$ represents the bernoulli distribution. Importantly, the choice of parameters for the beta distributions affect the behavior of the infection. If the underlying susceptibilities of all nodes are too high or the underlying infectivity of all nodes are too low, infection does not spread to each node in the dynamic graphs. Ideally, we want to avoid degenerate infection cases where infection does not spread. Thus, in order to generate our dynamic graphs we ensure infection spreads throughout the entire graph by the final timestep by choosing the parameters for the beta distributions appropriately. 

\begin{algorithm}
\caption{SI Model of Infectious Disease}\label{alg:cap}
\begin{algorithmic}[1]
\Procedure {SI}{$V_{0:T}$, $\mbf{A}_{0:T}$, $\mc{B}eta(\theta_{1,\text{inf}}$, $\theta_{2,\text{inf}})$, $\mc{B}eta(\theta_{1,\text{sus}}, \theta_{2,\text{sus}})$, $p_{\text{init}}$}
    \For {$i \in V_{0:T}$} \quad \Comment {Initialize infectivity parameter, susceptibility parameter, and $y_0(i)$ for all nodes}
        \If{$i$ has not been assigned parameters $\alpha$ and $\beta$}
            \State $\alpha_i \sim \mc{B}eta(\theta_{1,\text{inf}}$, $\theta_{2,\text{inf}})$
            \State $\beta_i \sim \mc{B}eta(\theta_{1,\text{sus}}, \theta_{2,\text{sus}})$
        \EndIf
    \EndFor
    \For {$i \in V_{0}$}
        \State $y_0(i) \sim B(p_{\text{init}})$
    \EndFor
    
    \For {$t \leftarrow 0, T$} \quad \Comment {For all timesteps, for all nodes in $V_t$, simulate SI infection}
        \For {$i \in V_{t}$}
            \If{$y_t(i) == 0$}
                \For {$j \in \mc{N}_{t}(i)$}
                    \State $y_t(i) \sim B(\alpha_j \cdot \beta_i \cdot \mbb{I}(y_{t}(j) = +1))$
                    \If{$y_t(i) == +1$}
                        \State \textbf{break}
                    \EndIf
                \EndFor
            \EndIf
        \EndFor
    \EndFor
    
    \For {$t \leftarrow 0, T$} \quad \Comment {Gather features for all nodes for all timesteps}
        \For {$i \in V_{t}$}
            \State $\mbf{x_t(i)} \gets [y_t(i), \alpha_i, \beta_i]$
        \EndFor
    \EndFor
    
    \State \textbf{return} $G_{0:T} = (V_t, \mbf{A}_t, \mbf{X}_t, \mbf{y}_t)$
\EndProcedure
\end{algorithmic}
\end{algorithm}

Below we provide additional details about each dataset used in the main paper including specific parameters for both synthetic and real-world networks. We begin with the synthetic datasets, each of which are generated in Networkx 3.2 \citep{hagberg2020networkx},
\begin{itemize}
    \item \textbf{Regular}: a 3-regular graph on $n=1000$ nodes.
    \item \textbf{Powerlaw}: an albert barabasi graph with $n=1000$ nodes grown by attaching new nodes each with $m=3$ edges that are preferentially attached to existing nodes with high degree.
    \item \textbf{Block}: a stochastic block model on $n=1000$ nodes with $c=20$ communities of equal size where a node forms an edge within its community with probability $p_{\text{in}}=0.10$ and an edge outside of its community with probability $p_{\text{out}}=0.001$. 
\end{itemize}

We provide the dataset details for our real datasets. For each real dataset, we discretize the dynamic graph by setting the static graph at a particular timestep as the tuple of nodes and edges that lie within the timestep's associated time interval. More specifically, the static graph at time $t$ is composed of all nodes and edges that lie within the interval [$tw$, $t (w + 1)$], where $w$ is the interval size. For each dataset, we list the specific choice of $w$.
\begin{itemize}
    \item \textbf{UCI}: a dynamic network of private messages sent on an online social network at the University of California Irvine. Nodes represent users, and edges represent private messages between users. There are a total of $\abs{\mc{E}}=59835$ temporal edges, and we set the interval size equal to $w=2000$.
    \item \textbf{Bitcoin}: a dynamic network of transactions on the Bitcoin platform Bitcoin OTC. Nodes represent users and edges represent transactions. There are a total of $\abs{\mc{E}}=35592$ temporal edges, and we set the interval size equal to $w=1000$.
    \item \textbf{Math}: a dynamic network of interactions on the stack exchange web site Math Overflow. Nodes represent users, and edges represent various types of interactions including answering questions, commenting on questions, and commenting on answers. There are a total of $\abs{\mc{E}}=506550$ temporal edges, and we set the interval size equal to $w=12000$
\end{itemize}

\begin{table}[t]
\centering
\footnotesize
\caption{Statistics for dynamic graphs. $N$ is the number of dynamic graphs provided in each of the train, validation, and test set.}
\begin{tabular}{ l c c c c } 
\toprule
\textbf{Dynamic graph} & \textbf{$\abs{\mc{V}}$} & \textbf{ $\abs{\mc{E}}$} & \textbf{$T$} & $N$\\ 
 \midrule
 \textbf{Regular} & 1,000 & 45,000 & 30 & 20 \\ 
 \textbf{Powerlaw} & 1,000 & 90,000 & 30 & 20 \\
 \textbf{Block} & 1,000 & 150,000 & 30 & 20 \\
 \midrule
 \textbf{UCI} & 1,899 & 59,835 & 29 & 20 \\
 \textbf{Bitcoin} & 5,881 & 35,592 & 35 & 20 \\
 \textbf{Math} & 24,818 & 506,550 & 20 & 20 \\
 \midrule
 \textbf{Higgs Day 1} & 3,124 & 16,781 & 24 & 1 \\
 \textbf{Higgs Day 2} & 12,075 & 127,662 & 24 & 1 \\
 \textbf{Higgs Day 3} & 23,853 & 304,255 & 24 & 1 \\
 \textbf{Higgs Day 4} & 30,822 & 459,707 & 24 & 1 \\
 \midrule
 \textbf{Gavin} & 2,541 & 140,040 & 36 & 1 \\
 \textbf{Ho} & 1,548 & 42,220 & 36 & 1 \\
 \textbf{Ito} & 2,856 & 8,638 & 36 & 1 \\
 \textbf{Uetz} & 922 & 2,159 & 36 & 1 \\
 \bottomrule
\end{tabular}
\label{tab:1}
\end{table}

\subsection{Real Datasets}

The Higgs dataset is a social network of twitter where nodes are users and edges indicate friendship and follower statuses. The dataset records tweets, replies, and mentions of the announcement of the discovery of the Higgs boson between users. As described in the main paper, we use four separate daily graphs each divided into 24 hours. At each timestep, the task is to predict which nodes become positive at the next timestep. At time $t$, we construct features for each node based on a learnable embedding matrix concatenated with the positive statuses of the nodes at timestep $t$. At time $t$, future labels for nodes are are the positive statuses of nodes at timestep $t+1$.

The protein-protein interaction networks are from a biological repository of dynamic protein-protein interaction networks (DPPIN) \citep{fu2022dppin}. As described in the main paper, we consider four separate dynamic graphs each spanning 36 timestamps where timesteps represent 3 successive metabolic cycles of yeast cells at different resolutions. We split the graph chronologically into 10 train graphs, 10 validation graphs, and 16 test graphs. At time $t$ we construct features for each node based on the node's protein type concatenated with the active statuses of the nodes at timestep $t$. At time $t$, future labels for nodes are the active statuses of the nodes at time $t+1$. We summarize statistics of all datasets in Table \ref{tab:1}.

\subsection{Model Details}

We provide a summary of all GNNs used in the main paper below. First, we describe the static GNNs in the context of homophily, including their aggregation scheme and additional designs that aim to alleviate specific heterophilous settings. Next, we describe the dynamic GNNs in the context of their dynamic components.
\begin{itemize}
    \item \textbf{SGC} (homophilous) uses symmetric-normalized mean aggregation without any intermediate weights or nonlinearities between layers \citep{wu2019simplifying}.
    \item \textbf{GCN} (homophilous) uses symmetric-normalized mean aggregation, including intermediate weights and nonlinearities between layers \citep{kipf2017semi}.
    \item \textbf{GIN} (homophilous) is theoretically more expressive than GCNs, leveraging sum aggregation \citep{xu2018how}.
    \item \textbf{GAT} (homophilous) uses self-attention aggregation \citep{velickovic2018graph}. GAT is a popular choice included in our experiments for completeness.
    \item \textbf{SAGE} (heterophilous) uses mean aggregation, but incorporates the heterophilous design choice of separating the ego-representations from aggregated neighbor representations \citep{hamilton2017inductive}.
    \item \textbf{GCNII} (heterophilous) uses symmetric-normalized mean aggregation, but incorporates the heterophilous design choice of the addition of residual connections \citep{li2018deeper}.
    \item \textbf{FA-GCN} (heterophilous) uses attention-based aggregation, but incorporates the heterophilous design choices of both SAGE and GCNII \citep{bo2021beyond}.
\end{itemize}

In our dynamic GNN experiments, we train the following dynamic GNNs described below in the context of their dynamic components.
\begin{itemize}
    \item \textbf{DGNN} combines a GCN and RNN by first applying the GCN to the initial node representations for all nodes across all timesteps. Next, an RNN is applied to each timeseries of node representations. Finally, a readout layer is applied to obtain the prediction for all nodes across all timesteps. This particular formulation for dynamic GCNs has been widely used in many dynamic graph applications \citep{manessi2020dynamic, narayan2018learning, chen2022gc}.
    \item \textbf{GCRN} combines a GCN and RNN by replacing each weight update in the RNN with a GCN. The new RNN layer, GCRN, is applied to each timeseries of node representations, and a readout layer is applied to obtain the the prediction for all nodes across all timesteps \citep{seo2018structured}.
    \item \textbf{EvolveGCN} combines a GCN and RNN by using an RNN evolve the weights of GCN layers across timesteps \citep{pareja2020evolvegcn}.
\end{itemize}

We implement and train our models on a GeForce GTX 1080. 

\subsection{Training and Evaluation Details}

We train all GNNs using the Adam optimizer and in full batch mode where each batch consists of a single static graph \citep{KingBa15}. We perform a hyperparameter search over number of layers in the range [0, 3], learning rate in the range [.1, .0001], and the size of hidden dimension in the range [32, 128]. Training is stopped when either 200 epochs are reached or when validation performance no longer improves after 50 epochs. We report the mean test AUROC and compute correlations over the top 4 models with the best validation AUROC.

We now describe the computation of static and dynamic homophily. At a particular timestep $t$, we compute static homophily as the mean over all static local homophily levels in the static graph at time $t$. Formally,
\eq{h_t^S = \frac{1}{\abs{V_t}}\sum_{i \in V_t} h_t^S(i) = \frac{1}{\abs{V_t}} \sum_{i \in V_t} \sum_{j \in \hat{\mc{N}}_t(i)} \frac{\mbb{I}[y_t(i) == y_t(j)]}{\abs{\hat{\mc{N}}_t(i)}},}

where $h_t^S(i)$ is the local static homophily for node $i$ defined as the ratio of neighbors with the same label as itself at time $t$. Following Definition 1, we compute dynamic homophily as the following: 
\al{h_t^D &= \frac{1}{2}\prn{\frac{1}{\abs{V_{t+1}^+}}\sum_{i \in V_{t+1}^+} h_t^D(i) + \frac{1}{\abs{V_{t+1}^-}}\sum_{i \in V_{t+1}^-} h_t^D(i)} \\&
= \frac{1}{2}\prn{\frac{1}{\abs{V_{t+1}^+}}\sum_{i \in V_{t+1}^+} \sum_{j \in \hat{\mc{N}}_t(i)} \frac{\mbb{I}[y_{t+1}(i) == y_t(j)]}{\abs{\hat{\mc{N}}_t(i)}} + \frac{1}{\abs{V_{t+1}^-}}\sum_{i \in V_{t+1}^-} \sum_{j \in \hat{\mc{N}}_t(i)} \frac{\mbb{I}[y_{t+1}(i) == y_t(j)]}{\abs{\hat{\mc{N}}_t(i)}}}.}

\section{ADDITIONAL PLOTS}\label{app:plots}

Below we include additional plots on all combinations of GNNs and dynamic graphs.



\begin{figure*}[h!]
     \centering
     \begin{subfigure}[b]{.49\textwidth}
         \centering
         \includegraphics[width=\textwidth]{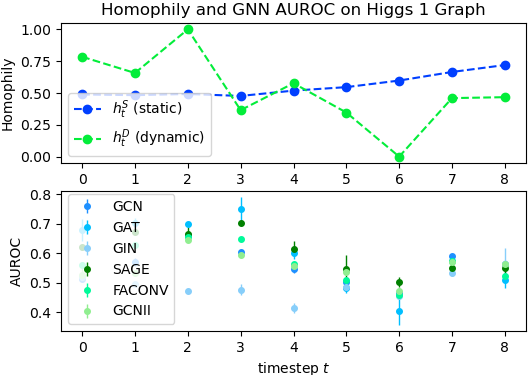}
         \caption{Higgs 1 dynamic graph}
     \end{subfigure}
     \hfill
     \begin{subfigure}[b]{.49\textwidth}
         \centering
         \includegraphics[width=\textwidth]{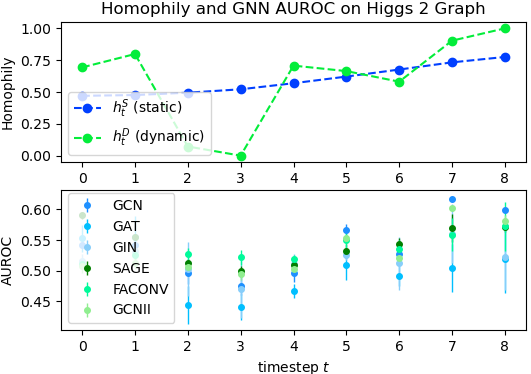}
         \caption{Higgs 2 dynamic graph}
     \end{subfigure}
\end{figure*}

\begin{figure*}[h!]
     \begin{subfigure}[b]{.49\textwidth}
         \centering
         \includegraphics[width=\textwidth]{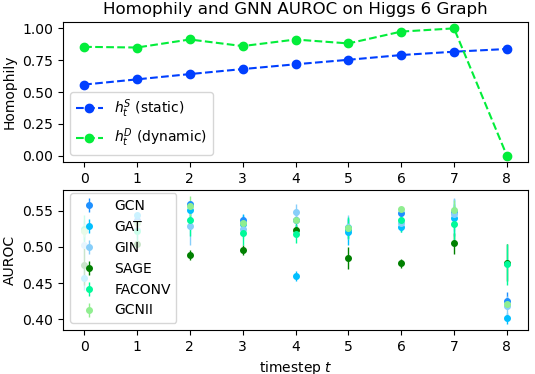}
         \caption{Higgs 6 dynamic graph}
     \end{subfigure}
     \hfill
     \begin{subfigure}[b]{.49\textwidth}
         \centering
         \includegraphics[width=\textwidth]{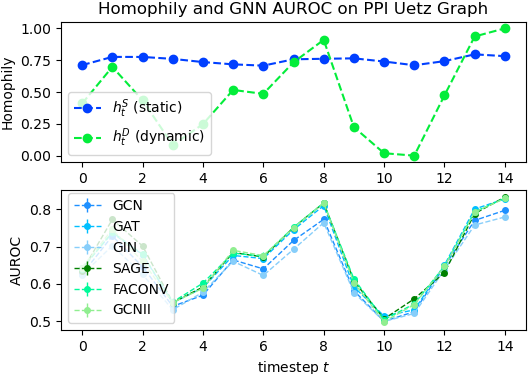}
         \caption{DPPIN Uetz dynamic graph}
     \end{subfigure}
\end{figure*}

\begin{figure*}[t!]
     \begin{subfigure}[b]{.49\textwidth}
         \centering
         \includegraphics[width=\textwidth]{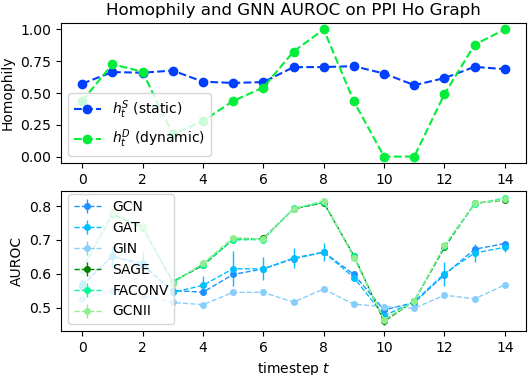}
         \caption{DPPIN Ho dynamic graph}
     \end{subfigure}
     \hfill
     \begin{subfigure}[b]{.49\textwidth}
         \centering
         \includegraphics[width=\textwidth]{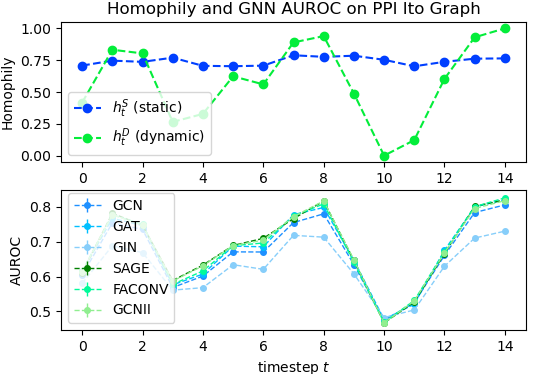}
         \caption{DPPIN Ito dynamic graph}
     \end{subfigure}
\end{figure*}

\section{SOCIETAL IMPACT \label{app:soc}}

As our work aims to understand the performance of GNNs in general dynamic node classification settings, we do not foresee any immediate \textit{negative} societal outcomes. For particular applications of dynamic node classification settings such as predicting the spread of infectious disease or misinformation, our work sheds light on current GNN limitations, potentially guiding the design of future GNNs aimed at better solving these tasks. 



\end{document}